\documentclass[conference,compsoc]{IEEEtran}

\ifCLASSOPTIONcompsoc
  \usepackage[nocompress,noadjust]{cite}
\else
  \usepackage{cite}
\fi

\let\citet\cite
\let\epsilon\varepsilon
\usepackage{amsmath,amssymb,amsfonts,amsthm}
\usepackage{graphicx}
\usepackage{algorithm,algpseudocode}
\usepackage{url,hyperref}
\usepackage[capitalise]{cleveref}
\usepackage{booktabs,multirow}
\usepackage{dsfont}
\usepackage{enumitem}
\usepackage[dvipsnames]{xcolor}

\theoremstyle{plain}
\newtheorem{theorem}{Theorem}[section]
\newtheorem{lemma}[theorem]{Lemma}
\newtheorem{proposition}[theorem]{Proposition}

\theoremstyle{definition}
\newtheorem{definition}[theorem]{Definition}

\theoremstyle{remark}
\newtheorem{remark}[theorem]{Remark}

\newcommand{\Xspace}{\mathcal{X}}
\newcommand{\Sspace}{\mathcal{S}}

\newcommand{\Rspace}{\mathcal{R}}
\newcommand{\E}{\mathbb{E}}
\newcommand{\KL}{D_\mathrm{KL}}

\newcommand{\Normal}{\mathcal{N}}
\newcommand{\Real}{\mathbb{R}}
\newcommand{\Tau}{\mathcal T}

\let\Tau H
\renewcommand{\tau}{h}

\newcommand{\Reald}{\mathbb{R}^d}
\newcommand{\Realdt}{\mathbb{R}^{d_t}}
\newcommand{\diag}{\operatorname{diag}}
\newcommand{\var}{\operatorname{Var}}

\newcommand{\repourl}{\url{https://github.com/zhxchd/pac_private_prediction}}

\DeclareMathOperator*{\argmax}{argmax}

\newcommand{\compactparagraph}[1]{
\par
\vspace{1ex}
\noindent{\bf #1.}
}

\let\paragraph\compactparagraph

\title{Private Prediction via PAC Privacy}

\author{
\IEEEauthorblockN{Xiaochen Zhu}
\IEEEauthorblockA{MIT CSAIL\\
Cambridge, MA, 02139\\
Email: xczhu@mit.edu}
\and
\IEEEauthorblockN{Mayuri Sridhar}
\IEEEauthorblockA{MIT CSAIL\\
Cambridge, MA, 02139\\
Email: mayuri@mit.edu}
\and
\IEEEauthorblockN{Srinivas Devadas}
\IEEEauthorblockA{MIT CSAIL\\
Cambridge, MA, 02139\\
Email: devadas@mit.edu}
}

\begin{document}
\bstctlcite{IEEEexample:BSTcontrol}

\maketitle

\begin{abstract}

Machine learning models are increasingly served behind APIs. This renders \emph{private prediction}, i.e., privatizing a model's outputs rather than its parameters, a natural privacy target: model outputs are lower-dimensional and far more stable to training-data changes than weights.

While differential privacy (DP) cannot effectively exploit this as it calibrates noise to worst-case sensitivity that is intractable to bound for non-convex models, we argue that PAC privacy is a natural fit for private prediction. It is instance-based, and calibrates noise to a black-box function’s empirical stability to control mutual-information (MI) leakage. The missing ingredient is efficient, adaptive composition. Serving predictions means answering a long stream of adaptively chosen queries from untrusted users; existing composition either fails under adaptivity, grows quadratically, or reverts to input-independent, DP-like noise. We close this gap with a new adversarial composition result via adaptive noise calibration and prove that MI accumulates only linearly under adaptive and adversarial querying.

Experiments across modalities show that prediction stability enables high utility even at a tiny per-query budget: on CIFAR-10, we achieve 87.79\% accuracy with a per-query MI budget of $2^{-32}$. This enables serving one million queries while provably bounding membership-inference success to 51.08\% -- the same guarantee as $(0.04, 10^{-5})$-DP. 
Further, in the presence of auxiliary public data, the large volume of PAC-private predictions enables us to distill a publishable model that can be queried without limit.
Concretely, 210,000 private labels on an ImageNet subset distill into a student reaching 91.86\% accuracy on CIFAR-10 with membership inference success bounded by 50.49\%, comparable to $(0.02, 10^{-5})$-DP.
\end{abstract}

\section{Introduction}\label{sec:intro}

The widespread adoption of machine learning (ML)
has raised significant privacy concerns, as models are often trained on sensitive data, such as medical records, financial transactions, and private communications~\cite{shokri2015ppdl,papernot2018sok,oecd2024aiprivacy}.
Studies have demonstrated that ML models, particularly over-parameterized neural networks, inadvertently memorize information about their training datasets; this information can be leaked, most notably through membership inference attacks (MIAs), where an adversary infers whether a specific individual's data was used to train the model~\cite{shokri2017membership,carlini2022membership}.

Differential privacy (DP)~\cite{dwork06dp,dwork2014algorithmic} has emerged as the de facto standard to address these concerns.
In ML, this is realized predominantly through \emph{private training}: algorithms like DP-SGD~\cite{song2013dpsgd,abadi2016deep} inject noise to gradients during optimization, yielding private model weights that are publishable under DP guarantees.
While theoretically sound, DP-trained models often face a harsh privacy-utility tradeoff:
acceptable utility typically necessitates a large privacy budget that substantially weakens the guarantee~\cite{jayaraman2019evaluating, demelius25,ertan2026fundamental}.

In practice, however, modern models are increasingly deployed as black-box services behind APIs, exposing only their \emph{predictions} in response to user queries, rather than weights~\cite{gan2023maas}.
This reality motivates an alternative paradigm known as \emph{private prediction}~\cite{dwork2018prediction}, where only model predictions are privatized and released.
Predictions are an attractive target on two counts. 
First, they are typically much lower-dimensional than parameters, sharply reducing the required noise for privacy. For instance, an ImageNet classifier carries millions of parameters but emits only a prediction over 1,000 classes~\cite{krizhevsky2012imagenet}.
Second, predictions are far more \emph{stable}, since small changes in the training data can substantially move the learned weights but barely perturb the outputs~\cite{breiman2001statistical,damour22underspecification}.

Realizing this promise under DP has proven difficult. DP calibrates noise to \emph{sensitivity}, i.e., the worst-case change in output over adjacent datasets. 
In general, tight calculation of sensitivity is NP-hard~\cite{xiao2008output}.
For prediction, tight sensitivity bounds are known only for restricted, well-behaved learners (e.g., convex ones) with analytically guaranteed stability, but not for general, non-convex models~\cite{dwork2018prediction}.
The principal workaround is sample-and-aggregate~\cite{nissim2007smooth}, exemplified by PATE~\cite{papernote2017pate,pate}, which aggregates the predictions of models trained on disjoint subsets of the data to force a bounded sensitivity by construction. 
However, this bound is enforced via the stability of the \emph{aggregation} of the predictions, rather than the predictor itself.
As a result, private prediction has seen limited adoption and can often underperform releasing a privately trained model~\cite{van2020trade}.
The shortfall, we argue, is not fundamental to private prediction, but instead due to the misalignment between DP's sensitivity-based noise calibration and the predictor's underlying stability.

Probably approximately correct (PAC) privacy~\cite{hanshen2023crypto,xiao2025pac} is a natural fit.
Being \emph{instance-based}, it measures the underlying stability of an arbitrary data processing function via \emph{black-box simulations}
and rewards stability automatically: stabler functions require significantly less noise to privatize.
The theoretical obstacle is \emph{adversarial composition}.
Unlike non-adaptive settings, untrusted users can act as adaptive adversaries and select queries based on the interaction history to maximize information leakage.
Yet standard composition~\cite{hanshen2023crypto} fails under adaptivity, and existing adversarial composition~\cite{hanshen2025thesis} either grows \emph{quadratically} in the number of releases or reverts to input-independent, DP-like noise that discards the stability advantage (cf. \cref{sec:pacpac}).

\vspace{-1ex}
\paragraph{Our Contributions}
We close this gap with \emph{posterior-aware adversarial composition}. Our core insight is that, to maintain tight privacy accounting against an adaptive adversary, the curator itself must \emph{adapt} to the adversary's evolving knowledge. We propose \emph{adaptive noise calibration}, where the curator maintains a belief state --- the posterior distribution of the secret given the interaction history --- and calibrates each release's noise to it. 
Under this mechanism we prove a new composition theorem in which mutual information accumulates only \emph{linearly} under adaptive, adversarial querying, retaining PAC privacy's instance-based utility with linear composition under adaptive and adversarial strategies.

This theoretical result enables the first application of PAC privacy to modern ML models: we release PAC-private predictions for model-agnostic classification tasks while provably limiting the success rate of arbitrary adversarial inference on training data. Across tabular, vision and NLP tasks our method sustains high accuracy while supporting millions of queries under tight privacy guarantees.
For example, our method achieves 87.79\% accuracy on CIFAR-10 while provably limiting MIA success under 51.08\% after one million queries, matching the MIA guarantee of $(0.04,10^{-5})$-DP.
To serve inference beyond a finite budget, we further use these cheap, high-volume predictions to label a public auxiliary set and distill a model that can be queried without limit, with a confidence filter that provably bounds the probability of retaining a data point mislabeled by private prediction.
On CIFAR-10, public access to a subset of ImageNet yields a distilled model reaching 91.86\% accuracy, with the same provable MIA guarantee as $(0.02,10^{-5})$-DP.
We believe that PAC privacy enables a path to private model release via private prediction, turning prediction stability into a publishable, privacy-preserving model.

\paragraph{Paper Organization} The rest of the paper is organized as follows. \cref{sec:background} reviews the PAC privacy framework. \cref{sec:composition} presents our main theoretical contribution: the posterior-aware composition framework and the linear adversarial composition theorem. In \cref{sec:private_responses}, we instantiate it to release PAC-private ML predictions, detailing the threat model, mechanism, and efficient implementation. \cref{sec:distillation} develops the private distillation protocol with confidence filtering. \cref{sec:experiments} presents our extensive empirical evaluation across modalities. \cref{sec:related} surveys related work, including differential privacy, and \cref{sec:conclusion} concludes the paper.

\section{Background}\label{sec:background}
\subsection{PAC Privacy}
PAC privacy \cite{hanshen2023crypto,xiao2025pac} provides instance-based privacy guarantees.
Assuming an input distribution,
it leverages entropy in private data, and allows automatic measurement and control of privacy leakage for arbitrary processing functions in a black-box manner. Formally, it is defined as follows:
\begin{definition}[$(\delta,\rho,P_S)$-PAC Privacy \cite{hanshen2023crypto}]\label[definition]{def:pac}
Let $\Sspace$ be the domain of the sensitive input.
Given a possibly randomized function $M : \Sspace \to \Rspace$, distribution $P_S$ over $\Sspace$, and a binary-valued attack success criterion $\rho : \Sspace\times \Sspace\to \{0,1\}$, we say that $M$ is $(\delta,\rho,P_S)$-PAC privacy if for any informed adversary $A:\Rspace\to\Sspace$, who knows $(P_S, M)$, takes $R=M(S)\in\Rspace$ as input, and outputs an estimate $\hat S$ of $S$, its success rate measured by the attack criterion $\rho$, denoted by $1-\delta_{A}$, is at most $1-\delta$, i.e., \begin{equation}
1-\delta_{A}
:=
\mathrm{Pr}
[\rho(\hat S,S)=1]
\leq1-\delta,
\end{equation} where the probability is taken over the randomness of $S\sim P_S$, $R\gets M(S)$, and $\hat S\gets A(R)$.
\end{definition}

\begin{remark}
\cref{def:pac} bounds at $1-\delta$ the winning probability of a security game in which the curator draws $S\sim P_S$ and releases $R=M(S)$, and the adversary, who knows $(P_S,M)$ but not $S$, outputs $\hat S$ and wins if $\rho(\hat S,S)=1$. This holds against a computationally unbounded adversary and extends naturally to a realistic one who does not know $P_S$. In ML, $P_S$ is the distribution of the training set and $M(S)$ is the trained model or its predictions.
\end{remark}

Rather than absolute success probability, PAC privacy is typically enforced by bounding the adversary's \emph{improvement} over their prior knowledge, i.e.,the \emph{posterior advantage}.

\begin{definition} [Posterior Advantage \cite{hanshen2023crypto}]\label[definition]{df:post_adv}
    For any adversary $A$ as in \cref{def:pac} with success rate $(1-\delta_A)$, its \emph{posterior advantage} under an $f$-divergence $D_f$ is
    $$\Delta_f\delta_{A}
    :=
    D_f(\mathrm{Bernoulli}(\delta_A)\|\mathrm{Bernoulli}(\delta_0)),$$%
    where $1-\delta_0$ is the optimal success rate for an \emph{a priori} adversary who only knows $P_S$ and $M$ without observing $M(S)$, i.e., $
    1-\delta_0 :=\max_{Q}\Pr_{S\sim P_S,\hat S\sim Q}[\rho(\hat S,S)=1].
    $
\end{definition}

Given $P_S$ and $\rho$, the optimal prior success rate $1-\delta_0$ is a known constant; for membership inference in which $P_S$ is specified such that each point is included in the training set with probability $50\%$, it is exactly $50\%$. With the prior fixed, bounding the posterior advantage $\Delta_f\delta_A$ bounds the final success rate $1-\delta_A$. The foundation of PAC privacy lies in the following connection between posterior advantage and the mutual information (MI) of the input and output.

\begin{theorem}[\cite{hanshen2023crypto}]\label{thm:post_adv_mi}
For any adversary $A$, its posterior advantage satisfies
$
\Delta_f\delta_{A}\leq \inf_{P_W}D_f(P_{S,M(S)}\,\|\, P_S\otimes P_W)
$, where $P_W$ ranges over distributions on $\Rspace$.
In particular, when $D_f$ is the KL-divergence and $P_W=P_{M(S)}$, we have \begin{equation}\label{eq:delta<=mi}
(1-\delta_{A})\log\frac{1-\delta_{A}}{1-\delta_0}+\delta_{A}\log\frac{\delta_{A}}{\delta_0}
\leq I(S; M(S)),
\end{equation}
where $I(\cdot;\cdot)$ denotes mutual information.
\end{theorem}

We note that this upper bound is \emph{attack-agnostic}: the mutual information $I(S; M(S))$ depends only on $P_S$ and $M$, and bounds the advantage for \emph{any} attack criterion $\rho$. Hence, one can privatize an algorithm against arbitrary adversaries simply by setting a mutual information budget $B$ and enforcing $I(S;M(S))\leq B$. Then, concrete guarantees on adversarial success rate can be derived via \cref{eq:delta<=mi}.

For a deterministic $M:\Sspace\to\Reald$, the following theorem determines the required scale of Gaussian noise that enforces a target MI budget, and hence, is PAC-private.

\begin{theorem}[Noise Determination~\cite{hanshen2023crypto,mayuri2025sp}]\label{thm:noise_cov}
    Given random variable $S\sim P_S$, a deterministic $M:\Sspace\to\Reald$ and an MI budget $B>0$, let $\var(M(S))=U\Lambda U^\top$ be the singular value decomposition (SVD) of the covariance matrix of $M(S)$, with $\Lambda=\diag(\lambda_1,\ldots,\lambda_d)$ and set $\Lambda_B$ as $$
    \Lambda_B=\diag\left(\frac{\sqrt{\lambda_i}\sum_{j=1}^d\sqrt{\lambda_j}}{2B}:i=1,\ldots,d\right).
    $$ 
    Then, $I(S;M(S)+Z)\leq B$ when $S\sim P_S$ and independently $Z\sim \Normal(0,U\Lambda_B U^\top)$.
\end{theorem}

\begin{remark}\label[remark]{rm:sampling_complexity}
\cref{thm:noise_cov} calibrates the Gaussian noise from the covariance of $M(S)$.
The covariance may be evaluated in its exact form, for example, when the secret space $\Sspace$ is tractable with size $m:=|\Sspace|$; indeed, for the instantiation considered in this paper, exact evaluation is feasible.
The cost is then dominated by evaluating $M(S)$ for all $S\in \Sspace$.
Once the $m$ outputs are collected, we optimize noise calibration by performing SVD directly on the weighted, centered $m\times d$ matrix in $O(md\cdot\min(m,d))$ time, avoiding explicit covariance formation. This decomposition also enables efficient noise sampling via linear transformation in $O(d^2)$ time. Consequently, the total complexity to calibrate and sample the noise is $O(mC+md\cdot\min(m,d)+d^2)$, where $C$ is the cost of a single evaluation of the function $M$. This is typically dominated by the evaluation term $mC$.

When $\Sspace$ is computationally intractable, one can \emph{estimate} the covariance matrix by Monte Carlo simulation, i.e., sampling $S$ from $P_S$ and evaluating $M(S)$ for $m'\ll m$ times in a black-box manner.
\citet{hanshen2023crypto} determines the sampling complexity $m'$ as a function of the budget $B$ and confidence $\gamma$. %
Alternatively, \citet{mayuri2025sp} uses the convergence of empirical variance as a practical criterion for sufficient sampling.

In both cases, the $m$ or $m'$ samples for noise calibration are independent of the realized secret $S$. Operationally, the curator samples a secret $S\sim P_S$, calibrates noise $\Sigma$ based on $P_S$, $M$, and $B$, and releases $M(S)+\Normal(0,\Sigma)$.
\end{remark}

\subsection{Membership Inference Attacks}\label{sec:bg_mia}

Membership inference attacks (MIAs)~\cite{carlini2022membership,shokri2017membership} are the de facto standard to empirically measure privacy leakage in ML. To provide a concrete understanding of the PAC privacy guarantees, we translate the MI budget $B$ into a guarantee on the MIA success rate. The standard MIA game can be formalized to match the PAC privacy framework as follows:

\begin{definition}[Membership Inference Attack~\cite{shokri2017membership,carlini2022membership}]
\label[definition]{def:mia}
Given a finite data pool $U=\{u_1, \cdots, u_n\}$, a processing mechanism $M$, and a distribution $P_S$ over subsets of $U$,
let $S\sim P_S$. 
An informed adversary $A$, knowing $(U, P_S, M)$ and observing $M(S)$, outputs a membership estimate $\hat{S} \subseteq U$. 
The \emph{individual membership inference success rate} for a user $u_i$ is defined as the probability that the adversary correctly determines their membership:
$$
1-\delta_{A, i} := \mathrm{Pr}_{S \sim P_S, \hat{S} \gets A(M(S))}[\mathbf{1}_{u_i \in S}=\mathbf{1}_{u_i \in \hat{S}}],
$$
where $\mathbf1_{E}$ is the indicator variable for event $E$.
\end{definition}

Membership inference for any user $u_i$ is then a specific instance of the attack criterion $\rho$ in \cref{def:pac}, so \cref{thm:post_adv_mi} applies directly.
Given an MI budget $B$, \cref{eq:delta<=mi} upper-bounds the adversary's success rate $(1-\delta_{A,i})$. In particular, if 
$P_S$ is specified such that $\Pr(u_i\in S)=1/2$ for all $i$ --- that is, every example in $U$ has a 50\% probability to be selected for training --- the prior success rate is $1-\delta_{0,i}=50\%$. Then a budget of $B=2^{-2}$ provably limits MIA success to 83.78\%; tightening to $B=2^{-10}$ reduces this bound to 52.21\%, and an extremely strict budget of $B=2^{-32}$ yeilds $\approx 50.001\%$, virtually indistinguishable from random guessing. These bounds apply simultaneously to every individual in $U$, since the prior success rate is uniform across the population.
For context, $(\epsilon,\delta)$-DP~\cite{dwork2014algorithmic} likewise provably bounds MIA success, by $1-(1-\delta)/(1+e^{\epsilon})$~\cite{mayuri2025sp}: $(1,10^{-5})$-DP limits MIA success to 73.11\%, while $(0.1, 10^{-5})$-DP reduces it to 52.50\%. \cref{fig:mi_to_mia} illustrates these upper bounds on MIA success rate across MI budgets and their corresponding DP parameters with matching MIA guarantees.
While the privacy notions differ semantically, we follow the literature~\cite{mayuri2025sp} to report these comparable DP parameters throughout the paper to serve as a familiar reference for our PAC privacy guarantees.
We defer more detailed discussions on DP to \cref{sec:related}.

\begin{figure}[t]
    \centering
    \includegraphics{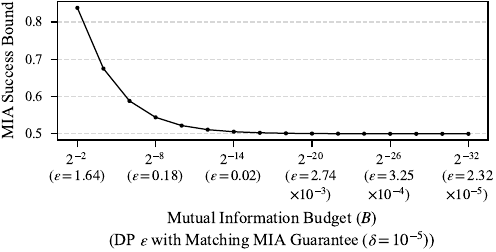}
    \caption{Provable upper bound on the individual MIA success rate as a function of the MI budget B, under a uniform prior success rate of 50\%. DP parameters with matching guarantees are showed in brackets for reference.}
    \label{fig:mi_to_mia}
\end{figure}

While we use MIA to demonstrate PAC privacy guarantees and follow the canonical literature~\cite{hanshen2023crypto,hanshen2024ccs,mayuri2025sp} to instantiate a subsampling-based input distribution --- $P_S$ ranges over subsets of $U$ --- the guarantee itself is MIA-specific. \cref{thm:post_adv_mi} is \emph{attack agnostic}: MI budget $B$ bounds the posterior advantage of \emph{any} inference attack on the training data, and can be cast to a concrete success bound after computing the prior success rate. Group MIA is one such case: its prior is lower than the individual 50\%, so the same budget $B$ yields a \emph{tighter} bound, whereas DP must resort to group privacy via composition. We discuss various privacy guarantees for our instantiation in \cref{sec:privacy_guarantees}.

\section{Adversarial Composition}\label{sec:composition}

The PAC privacy guarantees of \cref{sec:background} concern a single release; enforcing PAC privacy on a sequence of releases requires understanding how privacy leakage \emph{composes}. Furthermore, many practical settings are \emph{adaptive}: in stochastic gradient descent (SGD), each gradient step depends on the current parameters, hence on all prior outputs; and in private prediction, queries are additionally \emph{adversarial}, since a user may choose future queries based on observed outputs to maximize leakage. When the query sequence depends on the full interaction history, later releases can amplify information from earlier ones even if each release is individually PAC-private, and na\"ively composing per-step guarantees can lead to incorrect bounds on the total information leakage.

Prior composition in PAC privacy attains linear accumulation only under restrictive settings --- a sequence of non-adaptive, predetermined queries~\cite{hanshen2023crypto}, or independent resampling of a fresh secret at each step~\cite{hanshen2025thesis}. Neither fits private prediction, where queries are adaptive and adversarial and the secret training set \emph{persists} across releases. In this setting, the existing composition result either grows quadratically in $T$ or reverts to input-independent, DP-style noise~\cite{hanshen2025thesis}.

In this section, we formalize an adversarial composition setting for PAC privacy where the secret is persistent across releases and each mechanism may be adaptively chosen via a strategy \emph{unknown} to the data curator. We then present a composition algorithm that calibrates noise \emph{adaptively} to the interaction history and prove that the mutual information guarantee still accumulates \emph{linearly} despite adaptivity.

\subsection{Setting}\label{sec:composition_setting}
Following \cref{sec:background}, let random variable $S\in \Sspace$ represent the secret input data. Let $t=1,2,\ldots,T$ be discrete time steps. At initialization ($t=0$), the data curator samples the secret $S\sim P_S$.
At each time step $t$, an adversary, who wishes to gain knowledge about $S$, provides the data curator with a processing function $M_t:\Sspace\to\Rspace_t$, where $\Rspace_t=\Realdt$. In response, the data curator releases $R_t\in\Rspace_t$, a possibly noisy version of $M_t(S)$. We define random variable $\Tau_t=\{(M_i,R_i)\}_{i=1}^t$ as the full \emph{interaction history} up to time step $t$. It comprises both the adversary's chosen mechanisms and the data curator's releases, with $\Tau_0=\emptyset$. 

The adversary is \emph{adaptive}, whose choices are defined by a sequence of unknown strategies $\{F_t\}_{t=1}^T$.
At time step $t$, the adversary observes the history $\Tau_{t-1}$ and %
samples the next query $M_t\sim F_t(\Tau_{t-1})$, where $M_t:\Sspace\to\Realdt$.
Given history $\Tau_{t-1}$, the adversarial strategy $F_t$ defines a distribution over all possible functions from $\Sspace$ to $\Realdt$ and
is unknown to the data curator.
Our only assumption is that the adversary has no side-channel access to the secret: while a strong, informed adversary knows the distribution $P_S$ of the secret, it can learn about the realized $S$ only through the history.
In particular, the choice of the next query function $M_t$ depends on $S$ only via the history $\Tau_{t-1}$.
Formally, 
\begin{equation}\label{eq:adv_ass}
M_t\perp S \mid \Tau_{t-1}.
\end{equation}

We emphasize that the curator evaluates every $M_t$ on the \emph{same} secret input $S$ drawn at $t=0$, rather than resampling. This persistence is essential for meaningful privacy analysis.
Recall the previous MIA example where $P_S$ defines a distribution over subsets of $U$ s.t. $\forall u\in U \,P(u\in S)=1/2$.
In the context of ML, MIA adversaries are interested in knowing if a particular example $u\in U$ is \emph{ever} used across all the releases. If the curator re-samples at each step, the probability that an arbitrary $u$ is included in \emph{at least one} sampled subset rapidly approaches 1 as $T$ increases. A prior success rate near 100\% renders any bound on the posterior advantage vacuous.
By fixing $S$, we ensure that the input distribution of the overall mechanism remains $P_S$, and hence the prior MIA success rate remains constant, i.e., 50\%. 
Thus, any privacy loss is strictly due to leakage from releases, rather than increased prior success rates.
Persistent secret presents a greater challenge to bound the overall leakage compared to independent re-sampling: the adversary can aggregate information from the entire history $\Tau_T$ to refine their knowledge about the single, persistent secret $S$.

\subsection{Posterior-Aware Composition}\label{sec:pacpac}

Our goal is to bound the total mutual information leakage of the overall mechanism accumulated over $T$ steps, i.e., $I(S;R_1,\ldots, R_T)$. By \cref{thm:post_adv_mi}, this then bounds the posterior advantage and consequently, the adversarial success probability. By the chain rule of MI, the overall MI decomposes into the sum of conditional MIs: $$\textstyle I(S;R_1,\ldots, R_T)=\sum_{t=1}^T I(S;R_t\mid R_1,\ldots,R_{t-1}).$$ A natural approach is to bound each term by a per-step mutual information budget $b_t$. However, this is not trivial.

One straightforward attempt is to apply \cref{thm:noise_cov} at every step $t$ to calibrate noise scale $\Sigma_t$ for the mechanism $M_t$ such that $I(S;R_t)\leq b_t$ where $R_t=M_t(S)+\Normal(0,\Sigma_t)$. However, this is not sufficient as it fails to account for the dependencies between releases. For instance, when privatizing two adaptively chosen mechanisms, the total MI is $I(S;R_1,R_2)=I(S;R_1)+I(S;R_2\mid R_1)$, while \cref{thm:noise_cov} only enforces $I(S;R_1)\leq b_1$ and $I(S;R_2)\leq b_2$ independently for $S\sim P_S$. Even if the marginal leakage $I(S;R_2)$ is small, the conditional leakage $I(S;R_2\mid R_1)$
remains unbounded
because the adversary may choose $M_2$ adaptively to exploit the information revealed by $R_1$.

One theoretical workaround~\cite{hanshen2025thesis} is to enforce a stronger worst-case bound for each release, i.e., calibrate noise scale $\Sigma_t$ to bound $\sup_{Q} I_{S\sim Q}(S;M_t(S)+\Normal(0,\Sigma_t))\leq b'_t$ and $I(S;M_t(S)+\Normal(0,\Sigma_t))\leq b_t$ simultaneously for each $t$. 
The former constraint directly bounds the conditional MI given any possible history, effectively enforcing an \emph{input-independent, worst-case guarantee}, similar to that of local differential privacy~\cite{duchi2013ldp}.
However, as detailed in \cref{app:static_composition}, this forces a harsh trade-off: one must either accept a composition bound that scales \emph{quadratically} in $T$, or add excessive noise for a DP-like input-independent guarantee that ignores the instance-specific stability. This negates the core benefits of PAC privacy.

The fundamental issue of this \emph{posterior-oblivious composition} comes from \emph{noise misspecification}. For an informed adversary, its belief on the distribution of $S$ evolves (\cref{fig:belief_tracking}):
before observations, it knows $S$ is sampled from $P_S$ as a prior belief, often uniform over $\Sspace$.
After observing the history $\Tau_{t-1}$, a rational adversary would be tracking the posterior distribution $P_{S\mid \Tau_{t-1}}$ based on the observed history to refine their knowledge of the secret.
Calibrating noise to $P_S$ at every step therefore protects only against an adversary who has observed nothing, rather than one who has already learned from the interactions. 
To bound the conditional leakage $I(S;R_t\mid \Tau_{t-1})$ tightly, the noise must be calibrated \emph{adaptively} with respect to the \emph{current} posterior, effectively responding to the adversary's evolving knowledge.

To formalize this, we first define a helper function that determines the Gaussian noise required to bound the leakage of a specific processing function \emph{under a specific input distribution}. This extends the noise determination mechanism from \cref{thm:noise_cov} by treating the noise scale as a function of the input distribution as well.

\begin{definition}[Noise Calibration Function $\Sigma$]\label[definition]{def:noise}
A function $\Sigma:\mathcal P(\mathcal S)\times(\Reald)^\Sspace\times\Real_{>0}\to \Real^{d\times d}$ is said to be a noise calibration function if the following bound holds for any distribution $P$ over $\Sspace$, deterministic function $M:\Sspace\to\Reald$, and positive real number $B$: $$I_{S\sim P}(S;M(S)+\Normal(0,\Sigma(P,M,B)))\leq B.$$%
We can instantiate such a function $\Sigma$ using \cref{thm:noise_cov}.
\end{definition}

\begin{figure}[t]
\centering
\includegraphics{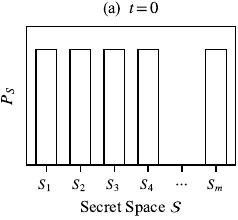}%
\hfill
\includegraphics{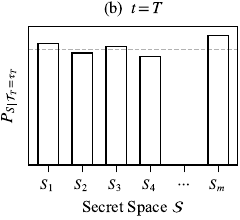}
\caption{An informed adversary refines their knowledge by tracking the posterior distribution of the secret conditioned on the observed history. In response, the curator adaptively calibrates noise to bound the conditional MI given the current history. Equivalently, this bounds the expected KL divergence between the posterior distribution and the prior $P_S$.}
\label{fig:belief_tracking}
\end{figure}

Equipped with this definition, we present an algorithm that allows the curator to achieve \emph{posterior-aware composition} via Bayesian posterior updates in \cref{alg:pacpac}.
The mechanism explicitly maintains a belief state $P_t$, which represents the posterior distribution of the secret $S$ given the current history $\Tau_t$.
At $t=0$, the curator samples the secret $S\sim P_S$ and initializes the belief state, identical to the prior distribution $P_0=P_S$. 
At each step, the adversary chooses a query $M_t$, via its querying strategy $F_t$ that is unknown to the curator.
Upon receiving $M_t$, the curator calculates the noise covariance $\Sigma_t=\Sigma(P_{t-1},M_t,b_t)$ required to satisfy the privacy budget for query function $M_t$ under the \emph{current belief} $P_{t-1}$.
Sampling the additive noise from $\Normal(0,\Sigma_t)$, the curator releases the noisy response $M_t(S)+\Normal(0,\Sigma_t)$, and then updates the belief state using Bayes' rule:
\begin{equation}\label{eq:p_upadte}
\begin{aligned}
&P_t(s) \propto  P_{t-1} (s)\\
&\quad\cdot \exp\left(-\frac{1}{2} \left(R_t - M_t(s)\right)^\top \Sigma_t^{-1} \left(R_t - M_t(s)\right)\right).
\end{aligned}
\end{equation}

\begin{algorithm}[t]
\caption{Posterior-Aware Composition}\label{alg:pacpac}
\begin{algorithmic}[1]
\State \textbf{Input:} Distribution $P_S$, per-step MI budgets $\{b_t\}_{t=1}^T$, noise calibration function $\Sigma$, horizon $T$. The adversary holds strategies $\{F_t\}_{t=1}^T$ unknown to the curator.
\State \textbf{Initialize:} 
\State \quad Initial belief $P_0 \leftarrow P_S$
\State \quad History $\Tau_0 \leftarrow \emptyset$
\State \quad Sample secret $S\sim P_S$
\State \textbf{for} $t=1,2,\ldots,T$ \textbf{do:}
    \State \quad \textbf{Adversary step:}
    \State \qquad Adversary samples $M_t:\Sspace\to\Realdt$ from $F_t(\Tau_{t-1})$
    
    \State \quad \textbf{Adaptive noise calibration:}
    \State \qquad $\Sigma_t \leftarrow \Sigma(P_{t-1}, M_t, b_t)$

    \State \quad \textbf{Release:}
    \State \qquad Release $R_t \leftarrow M_t(S) + \Normal(0,\Sigma_t)$
    
    \State \quad \textbf{Update:}
    \State \qquad Update history $\Tau_t \leftarrow (\Tau_{t-1}, (M_t,R_t))$
    \State \qquad Update posterior belief conditioned on $\Tau_t$: $\forall s\in\Sspace$,
    \State \qquad\quad $P_t(s) \propto \exp(-\frac{1}{2} (R_t - M_t(s))^\top \Sigma_t^{-1} (R_t - M_t(s)))\cdot P_{t-1} (s)$
\State \textbf{end for}
\end{algorithmic}
\end{algorithm}

We now analyze its privacy and complexity.

\subsubsection{PAC Privacy Guarantee} We show that under \cref{alg:pacpac}, the per-step MI bounds compose linearly to bound the overall MI. This matches the non-adaptive setting~\cite{hanshen2023crypto} despite adversarial adaptivity, as well as the resampling setting~\cite{sridhar2026pac} despite persistency. First, we prove that the belief state $P_t$ maintained by the algorithm via \cref{eq:p_upadte} correctly tracks the true posterior of $S$, which is the same distribution maintained by an informed and optimal adversary.

\begin{lemma}[Validity of Posterior Update]\label[lemma]{lm:post} Let $\Tau_t$ be the random variable representing the history up to time $t$ generated by \cref{alg:pacpac}. For any time step $0\leq t\leq T$ and any history $\tau_t$ realizable at $t$, when conditioned on $\Tau_{t}=\tau_{t}$, the belief state $P_t$ maintained by the curator is the same as the posterior distribution of the secret maintained by an informed and optimal adversary: $P_t=P_{S\mid\Tau_t=\tau_t}$.
\end{lemma}
\begin{proof} This is a Bayesian inference step. Full proof via induction is deferred to \cref{app:lm_proof}.
\end{proof}

We now prove the main result:
we bound the conditional MI at each step exactly by $b_t$, allowing linear composition.

\begin{theorem}[Adversarial Composition]\label{thm:composition}
Fix a prior $P_S$ over the secret space $\mathcal S$ and a horizon $T \in \mathbb{N}$.
Let $(R_1,\dots,R_T)$ be the sequence of releases produced by \cref{alg:pacpac}.
Then, the mutual information between the secret input $S$ and the sequence $(R_1,\ldots,R_T)$ released by \cref{alg:pacpac} satisfies
$
I(S ; R_1,\dots,R_T)\leq B_T := \sum_{t=1}^T b_t$.
\end{theorem}

\begin{proof}
We prove $I(S ; R_1,\dots,R_T)\leq I(S;\Tau_T)\leq B_T$. Full proof is deferred to \cref{app:composition_proof}.
\end{proof}

\begin{remark}

\cref{thm:composition} also bounds the expected KL divergence between the adversary's posterior belief of the secret's distribution and the prior distribution $P_S$: $
\E_{\tau_T}[\KL(P_{S\mid\Tau_T=\tau_T} \parallel P_S)]=I(S;\Tau_T)\leq B_T$.
Intuitively, this limits the adversary's power to refine its knowledge on $S$ and perform inference attacks on $S$ (\cref{fig:belief_tracking}).
\end{remark}

\cref{thm:composition} shows that when the noise calibration is adaptive to the interaction history, the per-step MI budgets $\{b_t\}$ compose linearly to a total bound of $B_T=\sum_{t=1}^T b_t$. Consequently, under a uniform budget allocation where $b_t=b$ for all $t$, the total MI after $T$ releases is upper bounded by $B_T=bT$. We translate these MI bounds into concrete attack-success guarantees in \cref{sec:privacy_guarantees}.

\begin{remark}[Linearity]
Our bound accumulates linearly rather than sublinearly as in advanced composition for approximate DP~\cite{dwork2010boosting}. This is intrinsic to controlling an \emph{information} quantity: MI is an expected KL divergence and composes \emph{additively} by the chain rule. %
Similar information-based guarantees also compose linearly:
R\'enyi-DP (RDP)~\cite{mironov2017renyi} and zero-concentrated DP (zCDP)~\cite{bun2016concentrated} do so in their own parameters, appearing sublinear only after conversion to $(\epsilon,\delta)$-DP. Linear accumulation is thus expected for our MI-based PAC privacy guarantee. In the persistent-secret, adaptive, and adversarial setting, it is a strict improvement over prior composition results in PAC privacy (\cref{app:static_composition}).
\end{remark}

\subsubsection{Computational Complexity}\label{sec:complexity}
While \cref{alg:pacpac} is theoretically valid for any secret space $\Sspace$, we first analyze its computational complexity when $m=|\Sspace|$ is finite.
As discussed in \cref{rm:sampling_complexity}, for a single release, noise calibration and sampling requires $O(mC+md\cdot\min(m,d)+d^2)$ time, dominated by the evaluation term $O(mC)$. For applications in this paper, $d<m$, and this cost becomes $O(mC+md^2)$.
\cref{alg:pacpac} introduces an additional step of posterior update for each release, which reuses the eigen-decomposition obtained during noise calibration.
This allows us to avoid explicit matrix inversion in \cref{eq:p_upadte} and compute the likelihoods via efficient vector projections in $O(md^2)$ time.
Consequently, the belief-tracking step in \cref{alg:pacpac} incurs \emph{no additional asymptotic overhead}, while being negligible compared to the $O(mC)$ cost for $m$ evaluations.

When $|\Sspace|$ is intractable (e.g., infinite), maintaining the full posterior distribution $P_t$ is not feasible. However, $P_t$ is only used for noise calibration to evaluate the covariance of the output distribution $M(S)$ under $S\sim P_{t-1}$. 
Thus, we can \emph{estimate} such covariance under $P_{t-1}$ via sampling (cf. \cref{rm:sampling_complexity}).
We leave the intractable case to future work.

\begin{figure*}[t]
\centering
\includegraphics[width=\linewidth,trim={2pt 3pt 0pt 0pt},clip]{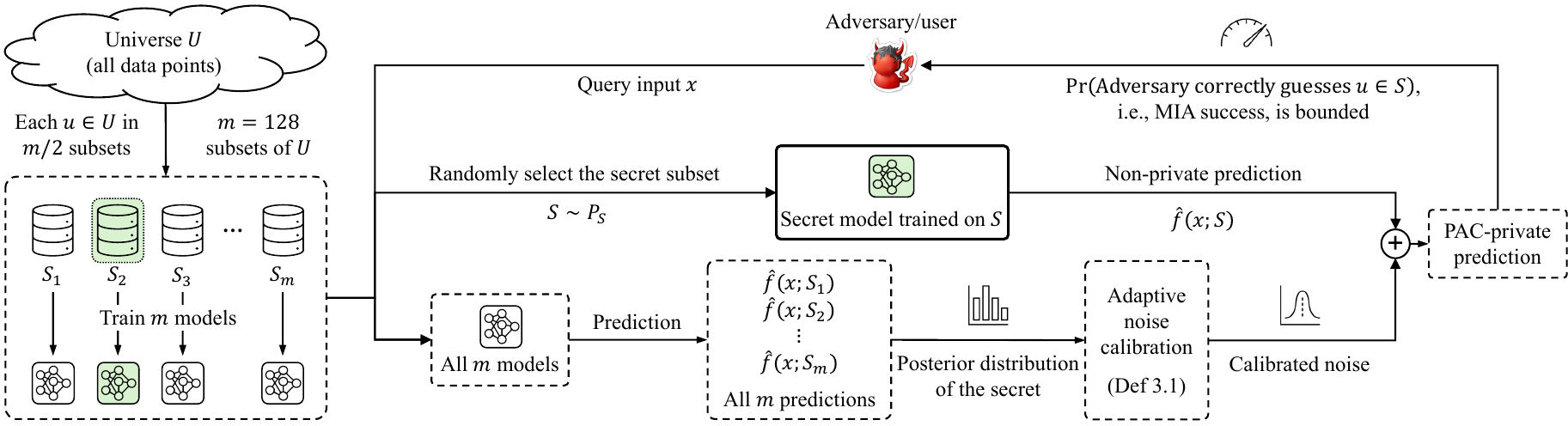}
\caption{An overview of our proposed framework to release PAC-private machine learning predictions to an adversary's queries.}
\label{fig:system}
\end{figure*}

\section{PAC-Private Machine Learning Predictions}\label{sec:private_responses}

While \cref{alg:pacpac} is applicable to a wide range of applications (such as SGD), we now focus on private prediction and instantiate \cref{alg:pacpac} for PAC-private predictions made by a ML model trained on sensitive data. \cref{fig:system} illustrates an overview for our proposed system.

\subsection{Setting}

Consider a supervised learning setting for a classification task with $d$ classes. Let $U =\{(x_i,y_i)\}_{i=1}^n\subset\mathcal X\times \mathcal Y$ denote the universe of all possible data points, where $\mathcal{X}$ is the input space and $\mathcal{Y} = \{1, \dots, d\}$ is the label space.
As discussed in \cref{sec:bg_mia}, we employ a \emph{subsampling-based} input distribution to enable provable defense against MIA, following the standard practice in PAC privacy ~\cite{hanshen2023crypto,hanshen2024ccs,mayuri2025sp}.
Specifically, the secret input $S$ is a random subset of $U$ drawn from a distribution $P_S$ over $\Sspace$, a collection of subsets of $U$. $P_S$ is specified such that for each individual data point $(x,y)\in U$, the probability that $(x,y)$ is included in $S$ is 50\%. This establishes a balanced membership inference game with a uniform prior success rate of 50\%. The sampled secret input dataset $S$ is then used to train an underlying ML model that generates predictions on user queries. We reiterate that the same secret set $S$ is used for all releases.

\subsubsection{Threat Model}
We assume a \emph{strong adversary} with unbounded compute. This adversary possesses full knowledge of the universe $U$, the secret space $\Sspace$, the prior distribution $P_S$, and the exact training and inference algorithms used by the curator, but not the realized secret dataset $S$.

At each step $t$, the adversary submits a query $q_t\in\mathcal X$, chosen adaptively from the interaction history and, in the worst case, adversarially to maximize what the responses reveal about $S$. The curator runs inference with the model trained on $S$ and responds with a prediction, potentially noised for PAC privacy. The adversary's goal is to infer the secret $S$: most canonically, the membership of a target $u^*\in U$, but also other objectives such as reconstructing a training example or recovering the trained model. We analyze these attacks in greater detail in \cref{sec:privacy_guarantees}.

We note that this threat model assumes a potent adversary with complete knowledge of the system except for the realized secret dataset while being computationally unbounded.
In practice, a realistic adversary will have bounded computational resources and is unlikely to know $P_S$ or the training/inference algorithms used by the curator.
Nevertheless, 
security established under the potent adversary model naturally extends to the realistic setting.

\subsection{Concrete Instantiation}

As established in \cref{sec:background}, PAC privacy bounds the success of an informed adversary by limiting MI between the secret input $S$ and the outputs. Under adaptive and adversarial mechanism selection, \cref{alg:pacpac} ensures the total leakage remains bounded.
We now instantiate this algorithm for PAC-private machine learning predictions.

\subsubsection{Input Distribution}\label{sec:input_discribution}
To instantiate \cref{alg:pacpac}, we must first specify the distribution $P_S$ from which the secret dataset $S$ is sampled.
A canonical choice for $P_S$ is the uniform distribution over all $2^n$ subsets of $U$, i.e., Poisson subsampling where each point in $U$ is selected independently with probability 0.5.
However, the exponential size of this support renders exact noise calibration and posterior belief update computationally intractable.
Instead, following prior work~\cite{mayuri2025sp}, we adopt a tractable construction where the support is restricted to a finite collection of $m$ subsets, i.e., $\Sspace=\{S_1,\ldots,S_m\}$. The prior $P_S$ is then defined as the uniform distribution over $\Sspace$.
While \citet{mayuri2025sp} generated $\Sspace$ via random half-partitioning to create $m/2$ pairs of disjoint sets, we introduce a novel construction to better approximate the canonical Poisson setting.
In particular, for each $u\in U$, we randomly select $m/2$ indices from $\{1,2,\ldots,m\}$ and assign $u$ to the corresponding subsets.
Crucially, this ensures that the marginal distribution of $S \sim P_S$ matches Poisson subsampling with $p=0.5$, when considering the randomness of $\mathcal S$.
The only distinction is that the security game is now played over a reduced support of size $m$,
rather than over an intractable space of all $2^n$ subsets.
While any even $m\geq2$ yields a 50\% balanced prior success rate for the MIA game, the choice of $m$ affects an informed adversary's prior for some other attacks (details deferred to \cref{sec:privacy_guarantees}). In particular, the value $1/m$ is a \emph{prior floor} for any attack, and we set $m=128$ to keep this rate below 1\% for an informed adversary who knows $\mathcal S$ and $P_S$.
This enables efficient and exact execution of \cref{alg:pacpac}.

\subsubsection{Mechanism Definition}\label{sec:mechnism_instantiation}
At time step $t$, the adversary observes the interaction history up to time $t$ and submits a query input $q_t \in \mathcal{X}$ to the curator. We formally define the resulting mechanism $M_t: \mathcal{S} \to \mathbb{R}^d$ as $M_t(\,\cdot\,) = \hat{f}(q_t; \,\cdot\,)$, where $\hat{f}(x; D)$ denotes the prediction on $x \in \mathcal{X}$ made by a model trained on dataset $D \subseteq \mathcal{X} \times \mathcal{Y}$, which is a $d$-dimensional vector corresponding to the $d$ classes.

Recall that \cref{thm:noise_cov} and \cref{def:noise} calibrate noise to privatize deterministic functions $M_t$; that is, the only source of randomness in the un-noised output $M_t(S)$ is the selection of the secret $S$ itself.
To achieve this, we explicitly derandomize the training and inference algorithms by fixing random seeds prior to execution. Note that these seeds can be made public without compromising the security game, as we consider a strong adversary that knows the learning and prediction algorithms, whose uncertainty lies solely in the realization of the secret $S$.

We emphasize that aside from being deterministic, we impose no restrictions on the underlying training or inference algorithms. Our framework treats training as a general mechanism $\hat f$, which is agnostic to the choice of model architecture or optimization strategy. Any supervised learning algorithm can be employed to define $\hat{f}$; crucially, this definition extends beyond simple model fitting. The function $\hat f$ may encapsulate an entire learning pipeline, including hyperparameter tuning, architecture search, or automated model selection, provided these steps rely solely on $S$.

\subsubsection{Computational Complexity}\label{sec:4_complexity}
With $P_S$ and $M_t$ defined, \cref{alg:pacpac} is fully instantiated. 
As discussed in \cref{sec:complexity}, a direct execution of the algorithm costs $O(mC+md^2)$ per query, where $C$ is the cost of evaluating $M_t(S)$ for a single $S\in \Sspace$.
In our setting, $M_t(S)=\hat f(q_t;S)$ involves training a model on dataset $S$ and then performing inference on $q_t$.
Training $m$ models \textit{for each query} is computationally prohibitive for interactive systems.

However, we can significantly optimize this process by moving the training phase offline. Since the secret space $\Sspace=\{S_1,\ldots,S_m\}$ is fixed at initialization, we can train all $m$ models $\{\hat f_1,\ldots,\hat f_m\}$ in advance as pre-processing, where $\hat f_i:\Xspace\to\Reald$ is the model trained on $S_i$. During online interaction, evaluating the covariance matrix of $M_t$ reduces to running inference on the trained models, bringing the online per-query complexity down to $O(mC_{\text{infer}}+md^2)$, where $C_{\text{infer}}$ is the cost of performing inference on a single model.
Further, the inference calls $\hat f_i(q_t)$ for $i=1,\ldots,m$ are embarrassingly parallel, reducing the per-query wall-clock latency to $O(C_{\text{infer}}+md^2)$ given sufficient computational resources. For context, inference costs $O(C_{\text{infer}})$ without privacy, so our PAC-private mechanism incurs an overhead of only $O(md^2)$ per query in terms of wall-clock latency, which is negligible compared to modern ML inference.

\subsubsection{Output Stabilization}
A classifier $\hat f$ can output \emph{soft} predictions (class probabilities, e.g.\ softmax) or \emph{hard} predictions (one-hot labels). Softmax scores carry more information but are unstable across models trained on slightly different data even when the predicted class agrees, so privatizing them demands larger noise and sacrifices utility. One-hot labels are far more stable: on an easy query most models agree, giving almost identical outputs across $S\in\Sspace$, near-zero output variance, and hence little noise to meet the privacy budget. We therefore privatize hard predictions: we add noise to the one-hot prediction to obtain $R_t$ and release $\argmax_i (R_t)_i$. This discrete structure also enables the confidence filter for distillation (\cref{sec:distillation}).

\subsection{Privacy Guarantees}\label{sec:privacy_guarantees}

\begin{table}[t]
\centering
\caption{PAC privacy guarantees on adversarial success rates (\%) under different priors after up to a million queries across different levels of per-query MI budget $b$.}
\label{tab:n_to_delta}
\setlength{\tabcolsep}{4.3pt}
\vspace{-.5\baselineskip}
(a) Prior = 1/2

\vspace{.5\baselineskip}
\begin{tabular}{lrrrrrrrr}
\toprule
$b$ & $2^{-4}$ & $2^{-8}$ & $2^{-12}$ & $2^{-16}$ & $2^{-20}$ & $2^{-24}$ & $2^{-28}$ & $2^{-32}$ \\
\midrule
$T=1$ & 67.5 & 54.4 & 51.1 & 50.3 & 50.07 & 50.02 & 50.00 & 50.00 \\
$T=10^{2}$ & 100 & 91.0 & 61.0 & 52.8 & 50.69 & 50.17 & 50.04 & 50.01 \\
$T=10^{4}$ & 100 & 100 & 100 & 76.9 & 56.89 & 51.73 & 50.43 & 50.11 \\
$T=10^{6}$ & 100 & 100 & 100 & 100 & 100 & 67.09 & 54.31 & 51.08 \\
\bottomrule
\end{tabular}

\vspace{\baselineskip}

(b) Prior = 1/128

\vspace{.5\baselineskip}

\begin{tabular}{lrrrrrrrr}
\toprule
$b$ & $2^{-4}$ & $2^{-8}$ & $2^{-12}$ & $2^{-16}$ & $2^{-20}$ & $2^{-24}$ & $2^{-28}$ & $2^{-32}$ \\
\midrule
$T=1$ & 5.6 & 1.7 & 1.0 & 0.8 & 0.79 & 0.78 & 0.78 & 0.78 \\
$T=10^{2}$ & 100 & 17.5 & 3.4 & 1.3 & 0.91 & 0.81 & 0.79 & 0.78 \\
$T=10^{4}$ & 100 & 100 & 63.8 & 9.4 & 2.28 & 1.10 & 0.86 & 0.80 \\
$T=10^{6}$ & 100 & 100 & 100 & 100 & 32.55 & 5.42 & 1.66 & 0.98 \\
\bottomrule
\end{tabular}

\end{table}

\begin{table}[t]
\centering
\caption{Number of PAC-private queries with per-query MI budget $b$ before the resulting upper bound on MIA success rate matches that of $(\epsilon,\delta)$-DP.}
\label{tab:dp_t}
\setlength{\tabcolsep}{3.3pt}
\begin{tabular}{lrrrrrrrr}
\toprule
$b$ & $2^{-4}$ & $2^{-8}$ & $2^{-12}$ & $2^{-16}$ & $2^{-20}$ & $2^{-24}$ & $2^{-28}$ & $2^{-32}$ \\
\midrule
$(1,10^{-5})$-DP & 1 & 28 & 454 & 7K & 116K & 2M & 30M & 477M \\
$(2,10^{-5})$-DP & 5 & 83 & 1K & 21K & 344K & 5M & 88M & 1B \\
$(4,10^{-5})$-DP & 9 & 154 & 2K & 40K & 632K & 10M & 162M & 3B \\
\bottomrule
\end{tabular}
\end{table}

Privacy guarantees for arbitrary attacks follow the same recipe. Cast the attack as a criterion $\rho$ (\cref{def:pac}) and let $\bar{\delta}_0=1-\delta_0$ be the optimal success rate of an informed adversary under our input distribution $P_S$. Running \cref{alg:pacpac} with per-step budget $b$ over $T$ releases bounds the total leakage by $I(S;R_1,\dots,R_T)\leq B_T=bT$ (\cref{thm:composition}), and \cref{eq:delta<=mi} translates MI bound $B_T$ under prior $\bar\delta_0$ into a bound on the posterior success rate $1-\delta_A$. Hence, every attack guarantee reduces to two numbers: its prior and the total budget. Under our construction of $P_S$ over $m$ subsets, the prior rate is a real number between $1/m$ and $1$.

\begin{proposition}[Prior floor]\label[proposition]{prop:floor}
Under $P_S$ uniform over $|\Sspace|=m$ subsets, for any attack where a correct guess of the secret constitutes success, the optimal prior success rate for an informed adversary $\bar\delta_0=1-\delta_0$ is at least $1/m$.
\end{proposition}

\begin{proof}
Fix some $S_0\in\mathcal S$. By \cref{df:post_adv}, $
\bar\delta_0
=\max_Q\Pr_{S\sim P_S,\hat S\sim Q}[\rho(\hat S,S)=1]
\geq \Pr_{S\sim P_S}[\rho(S_0,S)=1]\geq \Pr_{S\sim P_S}[S=S_0]=1/m.
$
\end{proof}

As noted earlier, we set $m=128$ to keep this prior floor below 1\%. However, for a realistic adversary who does not know the specific subsets in $\mathcal S$, it is intractable to reconstruct them from a high-entropy universe $U$, rendering the prior floor to be $\ll1/m$. While this potentially allows for a much smaller $m$ which benefits efficiency (cf. \cref{sec:4_complexity}), we adhere to the conservative choice of $m=128$.

\paragraph{MIA Guarantees} In our instantiation (\cref{sec:input_discribution}), each record in $U$ lies in exactly $m/2$ subsets,
so the MIA prior is 50\% for every record regardless of $m$. \cref{tab:n_to_delta}(a) presents the upper bounds of MIA success rate (or any attack with prior 50\%) after $T$ PAC-private releases with per-step MI budget $b$.
As $(\epsilon,\delta)$-DP also provably bounds MIA success (cf. \cref{sec:bg_mia}), \cref{tab:dp_t} reports how many releases the system sustains before its MIA guarantee matches that of $(\epsilon,10^{-5})$-DP.
Under a per-step budget of $b=2^{-8}$, MIA success is bounded by 54.42\% after one release (matching the MIA guarantee of DP with $\epsilon\approx0.18$), rising to 91\% ($\epsilon\approx2.31$) after a hundered. With a much tighter budget of $b=2^{-32}$, the privacy loss is almost negligible: even after \emph{one million releases}, MIA success is still bounded by 51.08\%, matching $\epsilon\approx0.04$ under DP. 
In fact, under this budget, the system supports \emph{$\approx477$ million releases} before reaching the same MIA guarantee as $(1,10^{-5})$-DP. This capacity exceeds the lifetime query volume of many deployments, rendering the privacy budget essentially inexhaustible.

\paragraph{Beyond MIA} The same recipe covers arbitrary attacks, which could have \emph{lower} priors and hence tighter bounds at a fixed budget. Group MIA over a $k$-tuple has a prior below $50\%$ that depends on the realized subsets $\mathcal S$; individual-example reconstruction~\cite{balle2022reconstructing} has prior $50\%$ under our sampling rate and can be lowered if records appear in fewer subsets~\cite{hanshen2024ccs}; and full training-set or model stealing~\cite{tramer2016stealing} has prior $1/m$, the floor. \cref{tab:n_to_delta}(b) gives the resulting upper bounds at the prior of $1/128$. Given the lower prior, attack success rate at $b = 2^{-32}$ after a million queries is $<$ 1\%.
We focus on MIA for the rest of the paper and defer other attacks' formulation and prior derivation to \cref{app:attacks}.

\section{Model Distillation from Private Predictions}\label{sec:distillation}

While \cref{alg:pacpac} enables a large number of queries (\cref{tab:n_to_delta}, \ref{tab:dp_t}), the privacy budget is still finite and will eventually be exhausted. When the accumulated leakage exceeds the pre-specified threshold, the system will be forced to halt. To enable unlimited queries, we distill a student model from the PAC-private predictions.
The quality of the distilled model depends on that of the input queries and the private labels; the latter depends on the stability of the predictions.

We can instantiate this given a dataset $D_{\text{pub}}\subseteq\Xspace$.
It does not need to be labeled, and may consist of public data or a pool of user queries collected while serving the private predictions; the privatized labels on $D_{\text{pub}}$ are used for distillation.
The distilled student inherits the rigorous PAC privacy via post-processing, and can thus be released for unlimited inference without incurring further privacy loss.

However, noise introduces label errors that can degrade the student model's utility.
To mitigate this, we propose a confidence filtering step.
Standard practice in model distillation filters samples based on teacher's confidence scores~\cite{lee2013pseudo, xie2020selftraining, sohn2020fixmatch}.
While our mechanism only responds with a noisy predicted label,
we rely on the structure of the underlying non-private prediction as a one-hot vector in $\Reald$ and the public noise distribution (i.e., $\Sigma$). This allows us to statistically test possible predictions against the observed, noisy one. This step does not incur privacy leakage, as PAC privacy guarantees hold even when the adversary knows the privatized mechanism, which includes the noise distribution.

Given $x\in D_\text{pub}$, let $e_y=\hat f(x;S)$ be the non-private prediction, where $e_y\in\Reald$ is the $y$-th standard basis vector. Let $r=e_y+\Normal(0,\Sigma)$ be the observed private prediction and $\tilde y=\argmax_i r_i$.
To ensure label quality, we retain the sample $(x, \tilde{y})$ to train the student model only if the likelihood of $\tilde{y}$ significantly exceeds that of \emph{any other class} $j \neq \tilde{y}$. Formally, we use the following test to ensure that a mislabeled sample is retained with probability at most $\alpha$.

\begin{proposition}\label[proposition]{prop:filtering}
    Given any significance level $\alpha \in (0, 1)$, we retain the sample $(x,\tilde y)$ if and only if for all $j\neq \tilde y$:$$
    T_j(r)=\frac{(e_{\tilde y}-e_j)^\top\Sigma^{-1}(r-e_j)}{\sqrt{(e_{\tilde y}-e_j)^\top\Sigma^{-1}(e_{\tilde y}-e_j)}}\geq \Phi^{-1}(1-\alpha),
    $$
    where $\Phi^{-1}$ is the inverse CDF of $\Normal(0,1)$.
    Then, the probability that $(x,\tilde y)$ is retained while $\tilde y\neq y$ is at most $\alpha$.
\end{proposition}

\begin{proof}
$r\sim\Normal(e_y, \Sigma)$ with unknown $e_y$ and known $\Sigma$. For $j\neq\tilde y$, we test $e_y=e_j$ vs $e_y=e_{\tilde y}$. Test $T_j(r)\geq\Phi^{-1}(\alpha)$ has size $\alpha$.
Full proof is deferred to \cref{app:filter_proof}.
\end{proof}

\cref{prop:filtering} upper-bounds label noise: regardless of the ground-truth non-private prediction, the probability of falsely validating a private label altered by noise is at most $\alpha$. This ensures that the student model is trained exclusively on high-confidence samples where the private prediction's signal is strong enough to statistically rule out all others.

\section{Evaluation}\label{sec:experiments}
\subsection{Datasets and Models}
We evaluate our framework on six standard classification benchmarks spanning tabular, image, and text modalities.
\begin{itemize}[leftmargin=*]
\item \textbf{Census Income}~\cite{adult} is derived from the 1994 US Census. The task is to predict if an individual's annual income exceeds \$50K from 15 demographic attributes.
\item \textbf{Bank Marketing}~\cite{bank} contains marketing data of a Portuguese bank. The goal is to predict if a client will subscribe to a term deposit based on 20 features.
\item \textbf{CIFAR-10} and \textbf{CIFAR-100}~\cite{cifar} are standard benchmarks consisting of $32\times32\times3$ images of 10 and 100 classes.
\item \textbf{IMDb Reviews}~\cite{imdb} is a text dataset of movie reviews, and the task is to classify a review as positive or negative.
\item \textbf{AG News}~\cite{agnews} is a text dataset comprising news articles and the task is to classify articles into four categories.%
\end{itemize}

For the tabular datasets, we construct a random 80/20 train-test split. For the image and text datasets, we retain the original train-test splits.
Dataset statistics are summarized in \cref{tab:dataset_summary}.
The training split constitutes the universe $U$, and the secret training dataset $S$ is sampled from $U$ according to the input distribution $P_S$ described in \cref{sec:input_discribution}.

\begin{table}[t]
    \caption{Summary of datasets used in our experiments.}
    \label{tab:dataset_summary}
    \centering
    \setlength{\tabcolsep}{7pt}
    \begin{tabular}{llcccc}
    \toprule
    Dataset & Modality & Train & Test & Classes\\
    \midrule
    Census Income~\cite{adult} & Tabular & 39,073 & 9,769 & 2\\
    Bank Marketing~\cite{bank} & Tabular & 32,950 & 8,238 & 2 \\
    \midrule
    CIFAR-10~\cite{cifar} & Image & 50,000 & 10,000 & 10\\
    CIFAR-100~\cite{cifar} & Image & 50,000 & 10,000 & 100\\
    \midrule
    IMDb Reviews~\cite{imdb} & Text & 25,000 & 25,000 & 2\\
    AG News~\cite{agnews} & Text & 120,000 & 7,600 & 4\\
    \bottomrule
    \end{tabular}
\end{table}

\begin{table*}[t]
\setlength{\tabcolsep}{4.4pt}
\centering
\caption{Average test accuracy (\%) of PAC-private predictions under various per-step MI budgets $2^{-32}\leq b\leq 2^{-4}$ across datasets. See \cref{tab:n_to_delta} on the provable upper bounds on attack success for each $b$ after varying numbers of releases. For all datasets, even when the per-step MI budget $b$ is as tight as $2^{-32}$, the private predictions maintain strong utility.} %
\begin{tabular}{llcccccccccc}
\toprule
Modality & Dataset & Non-Private & $b=\infty$ & $b=2^{-4}$ & $b=2^{-8}$ & $b=2^{-12}$ & $b=2^{-16}$ & $b=2^{-20}$ & $b=2^{-24}$ & $b=2^{-28}$ & $b=2^{-32}$ \\
\midrule
\multirow{2.5}{*}{Tabular} & Census Income & 87.39 & 87.17 & 87.15 & 86.68 & 85.92 & 85.86 & 85.84 & 85.84 & 85.84 & 85.84 \\
\cmidrule{2-12}
 & Bank Marketing & 91.98 & 91.69 & 91.67 & 91.02 & 90.36 & 90.28 & 90.28 & 90.27 & 90.27 & 90.29 \\
\midrule
\multirow{2.5}{*}{Image} & CIFAR-10 & 97.37 & 95.80 & 95.71 & 93.52 & 88.35 & 87.90 & 87.81 & 87.80 & 87.80 & 87.79 \\
\cmidrule{2-12}
 & CIFAR-100 & 84.02 & 77.79 & 77.69 & 75.56 & 58.38 & 56.34 & 56.17 & 56.13 & 56.10 & 56.11 \\
\midrule
\multirow{2.5}{*}{Text} & IMDb Reviews & 87.10 & 85.13 & 85.13 & 84.46 & 74.26 & 69.32 & 69.16 & 69.09 & 69.10 & 69.10 \\
\cmidrule{2-12}
 & AG News & 91.61 & 90.44 & 90.35 & 87.95 & 80.18 & 79.42 & 79.31 & 79.27 & 79.25 & 79.25 \\
\bottomrule
\end{tabular}

\label{tab:b_to_acc}
\end{table*}

\begin{figure*}[t]
\centering
\includegraphics[trim={0 1.2em 0 1.8em},clip]{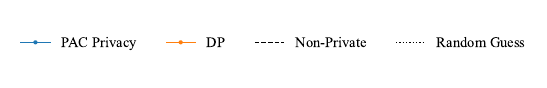}

\includegraphics{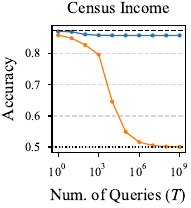}%
\hfill
\includegraphics{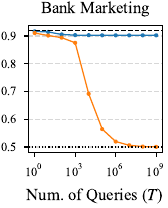}%
\hfill
\includegraphics{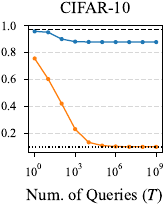}%
\hfill
\includegraphics{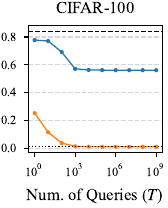}%
\hfill
\includegraphics{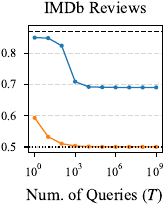}%
\hfill
\includegraphics{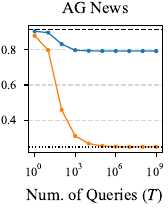}
\caption{
Accuracy vs the number of supported queries $T$, with both DP and PAC privacy configured to provide the MIA guarantee of $(1,10^{-5})$-DP after $T$ queries. DP private prediction degrades to random guessing as $T$ grows, while PAC privacy maintains strong utility.}
\label{fig:pac_vs_dp}
\end{figure*}

Since our privacy framework is model-agnostic, we can employ any pipeline to train the model on the secret dataset $S$ to instantiate $\hat f(\,\cdot\,;S)$, as discussed in \cref{sec:mechnism_instantiation}.
For tabular datasets, we utilize XGBoost~\cite{chen2016xgboost}. To optimize performance, we perform a hyperparameter sweep and choose the best combination via cross-validation locally on $S$. 
For image classification, we use the Wide-ResNet-28-10 architecture~\cite{zagoruyko2016wide}, a standard choice in recent DP literature~\cite{ de2022unlocking,tang2025differentially}. This model has 36.5 million parameters and is designed for CIFAR-style $32\times32\times3$ images. We train the model using standard data pre-processing and hyperparameters. Additionally, as $S$ is a subsampled subset of $U$ with expectedly half the size, we apply CutMix~\cite{yun2019cutmix} and MixUp~\cite{zhang2018mixup} for augmentation.
For text classification, we train a BERT-Small model~\cite{devlin2019bert, turc2019well} from scratch, which is a compact variant of BERT with 28.8 million parameters. We reserve 10\% of $S$ as a validation set for early stopping. 
Our code is available at \repourl, and implementation details are provided in \cref{app:implementation_details}.

\subsection{PAC-Private Predictions}\label{sec:exp:ppp}

We experiment with the private prediction mechanism described in \cref{sec:private_responses} on each of the six datasets. As a one-time pre-processing step, we sample $m=128$ subsets to fix the prior $P_S$ and train a model on each subset.
We then run 1,000 independent trials.
In each trial, we submit a random permutation of test examples to query our private prediction mechanism and report the average accuracy.

We test our mechanism under a wide range of per-step MI budgets $b$, ranging from $2^{-4}$ down to $2^{-32}$. In this interactive setting, we deliberately do \emph{not} fix a total MI budget $B$ to be allocated over all test examples,
because doing so would implicitly tie the overall privacy guarantee to the size of the test data.
Instead, our goal is to support a large number of queries, potentially far exceeding the size of the test dataset, and the overall privacy guarantee depends on the total number of queries $T$ the system aims to support.
Therefore, we treat $b$ as the primary security parameter: the cumulative privacy guarantee follows as a function of $b$ and the total number of queries $T$. A data curator can consult \cref{tab:n_to_delta} to select an appropriate value of $b$ based on their expected query volume and desired privacy target.
For example, to bound MIA success by 51.08\% after one million queries, one can set $b=2^{-32}$ (\cref{tab:n_to_delta}(a)).
In \cref{sec:mia_experiments}, we further derive an optimal MIA decision rule and report its success rates.
Additionally, we include two non-private baselines.
First, we run our private prediction mechanism with $b$ set to infinity. In this case, the predictions are still made by the model trained on the secret dataset $S$ sampled from $P_S$, but no noise is added.
Finally, we also report the non-private baseline for a model directly trained on the full universe $U$. The results are presented in \cref{tab:b_to_acc}.

\begin{table*}[t]
\centering
\caption{Upper bounds of MIA success and the average MIA accuracy of \cref{prop:optimal_mia} after $T$ PAC-private predictions on CIFAR-10 at varying per-step MI budget $b$. Empirical MIA accuracy is strictly bounded by the theoretical guarantee.}
\label{tab:mia}
\setlength{\tabcolsep}{6.3pt}
\begin{tabular}{llcccccccccc}
\toprule
\multicolumn{2}{l}{Number of Queries $T$} & 100,000 & 200,000 & 300,000 & 400,000 & 500,000 & 600,000 & 700,000 & 800,000 & 900,000 & 1,000,000 \\
\midrule
\multirow{2}{*}{$b=2^{-12}$} & MIA Bound (\%) & 100.00 & 100.00 & 100.00 & 100.00 & 100.00 & 100.00 & 100.00 & 100.00 & 100.00 & 100.00 \\
 & MIA Acc (\%) & 74.17 & 90.00 & 97.14 & 99.55 & 99.95 & 100.00 & 100.00 & 100.00 & 100.00 & 100.00 \\
\midrule
\multirow{2}{*}{$b=2^{-16}$} & MIA Bound (\%) & 100.00 & 100.00 & 100.00 & 100.00 & 100.00 & 100.00 & 100.00 & 100.00 & 100.00 & 100.00 \\
 & MIA Acc (\%) & 53.02 & 54.34 & 55.68 & 57.18 & 58.23 & 59.74 & 61.31 & 62.44 & 64.18 & 64.86 \\
\midrule
\multirow{2}{*}{$b=2^{-20}$} & MIA Bound (\%) & 71.48 & 79.85 & 85.89 & 90.60 & 94.38 & 97.39 & 99.61 & 100.00 & 100.00 & 100.00 \\
 & MIA Acc (\%) & 50.75 & 50.99 & 51.19 & 51.47 & 51.76 & 52.00 & 52.13 & 52.28 & 52.39 & 52.49 \\
\midrule
\multirow{2}{*}{$b=2^{-24}$} & MIA Bound (\%) & 55.45 & 57.71 & 59.43 & 60.87 & 62.15 & 63.29 & 64.34 & 65.32 & 66.23 & 67.09 \\
 & MIA Acc (\%) & 50.09 & 50.17 & 50.35 & 50.25 & 50.24 & 50.36 & 50.40 & 50.41 & 50.47 & 50.47 \\
\midrule
\multirow{2}{*}{$b=2^{-28}$} & MIA Bound (\%) & 51.36 & 51.93 & 52.36 & 52.73 & 53.05 & 53.34 & 53.61 & 53.86 & 54.09 & 54.31 \\
 & MIA Acc (\%) & 49.96 & 50.03 & 50.07 & 50.08 & 50.01 & 49.96 & 49.99 & 50.02 & 50.01 & 50.00 \\
\midrule
\multirow{2}{*}{$b=2^{-32}$} & MIA Bound (\%) & 50.34 & 50.48 & 50.59 & 50.68 & 50.76 & 50.84 & 50.90 & 50.97 & 51.02 & 51.08 \\
 & MIA Acc (\%) & 50.05 & 50.00 & 49.88 & 49.85 & 49.86 & 49.88 & 49.90 & 49.94 & 49.90 & 49.85 \\
\bottomrule
\end{tabular}
\end{table*}

\paragraph{Privacy-Utility Trade-off} 
Utility loss in our framework originates from two sources: subsampling and noise injection. 
That is, to enable PAC privatization, our mechanism trains the underlying model on a secret dataset $S$ subsampled from $U$.
This results in a \emph{subsampling} error, captured by the performance gap between the non-private and $b=\infty$ baselines in \cref{tab:b_to_acc}.
For all datasets, this gap is minimal ($<2\%$), with the only exception of CIFAR-100, which is more sensitive to subsampling due to the small number of samples per class (i.e., merely 500). Second, noise injection further degrades the utility as the budget $b$ tightens.

We observe a surprising phenomenon: rather than collapsing to triviality, the utility converges even as $b$ tightens to extremely small values.
As shown in \cref{tab:b_to_acc}, accuracy remains nearly constant as $b$ decreases from $2^{-16}$ to $2^{-32}$, and our mechanism maintains strong utility even at $b=2^{-32}$ for all datasets.
This resilience is a direct consequence of the instance-based nature of PAC privacy.
When the predictions are stable --- that is, when most of the underlying models trained on different $S\in\Sspace$ agree on the prediction for a given input --- the MI is inherently low.
Consequently, the mechanism can satisfy extremely strict per-step privacy bounds with minimal noise, preserving strong utility even when $b$ is vanishingly small.
Notably, this is only true with hard predictions, i.e., one-hot vectors. Confidence scores are often much less stable, as discussed in \cref{app:results}.%

An effective local proxy for such stability is the subsampling error.
Small subsampling error implies \emph{generalizability} --- that is, the subsampled dataset is representative of the population, and the model trained on a subsampled dataset generalizes well beyond its training data.
This further suggests that models trained on \emph{different} subsets are also likely to converge to similar decision boundaries, thereby exhibiting high stability and requiring less noise to hide the secret set.
This correlation is evident in our results.
For the two tabular datasets, where the tasks are simpler and data is abundant, the subsampling gap is negligible ($<0.3\%$). Consequently, the predictions are highly stable across the $m$ models, and we observe almost no utility drop at $b=2^{-32}$ as compared to the non-private baseline. 
In contrast, CIFAR-100 exhibits a larger subsampling penalty from 84.02\% to 77.79\%.
This indicates larger variance and thus, noise for privatization, leading to larger utility drop at $b=2^{-32}$ (down to 56.11\%). Nevertheless, this performance remains competitive.
Crucially, this suggests that a model that generalizes well and exhibits low variance will naturally incur a lower utility cost under PAC privacy, alluding to ``win-win'' scenarios between robust learning and privacy.

\paragraph{Comparison to DP Prediction}
While the sensitivity of a non-convex model's predictions is intractable to bound, the canonical DP approach is \emph{sample-and-aggregate}~\cite{nissim2007smooth}, exemplified by PATE~\cite{pate,papernote2017pate}: the training data is split into $K$ \emph{disjoint} subsets, where a model is trained on each, and queries are answered by a noisy aggregate of their votes --- sensitivity is thus bounded by construction.
We implement this approach using GNMax~\cite{papernote2017pate} and use Gaussian DP~\cite{dong2022gaussian} for composition, which is exact for Gaussian mechanisms and tighter than advanced $(\epsilon,\delta)$-DP composition~\cite{dwork2010boosting}, RDP~\cite{mironov2017renyi}, or zCDP~\cite{bun2016concentrated}.
We additionally grant the DP baseline the best number of subsets $K$ at no privacy cost.
We configure both DP and ours to satisfy the MIA guarantee of $(1,10^{-5})$-DP after $T$ queries --- that is, MIA success bounded by 73.11\%, and a total MI budget of $\approx0.11$ for PAC privacy.
We vary $T$ and report the results in \cref{fig:pac_vs_dp}.

Across all datasets, DP prediction degrades to random guessing as $T$ grows, remaining tolerable only at lower query volumes on the simpler tasks.
This is due to the \emph{input-independent} aggregation sensitivity:
under a fixed budget, when $T$ is large, the required per-query $\epsilon$ vanishes and DP must converge to random guessing to satisfy the worst-case, input-independent guarantee.
As a result, DP private prediction is typically used to label a modest set of public examples for distillation~\cite{papernote2017pate,pate}, not as a standalone high-volume online prediction service.
Unlike DP, PAC privacy exploits the stability of the underlying predictor, adding negligible noise when the prediction is \emph{stable} across possible training datasets.
Thus, accuracy plateaus at a strong, non-trivial level even as $T$ grows.

\paragraph{Comparison to DP Training}
As DP prediction does not scale, the standard DP approach is to privatize model \emph{weights} via DP-SGD~\cite{song2013dpsgd,abadi2016deep} during training~\cite{van2020trade}.
The state of the art for private image classification achieves 72.32\% on CIFAR-10 at $\epsilon=1$ and 43.33\% on CIFAR-100 at $\epsilon=3$~\cite{tang2025differentially};
ours attains 87.79\% and 56.11\% while supporting ${\approx}$477M queries before reaching the MIA guarantee matching $\epsilon=1$ and over 2 billion before $\epsilon=3$. 
Our improvements are driven by the fact that we privatize \emph{predictions} behind an API rather than publishing weights, and PAC privacy is able to exploit the stability of those predictions.
While we do not replace DP-SGD since we do not release the model or serve infinite queries,
we offer a markedly more favorable privacy-utility trade-off for ML inference services where the threat model aligns.
We revisit the comparison to DP training in \cref{sec:distillation_experiments}.

\begin{table*}[t]
\centering
\caption{CIFAR-10 test accuracy of a model distilled from $T$ PAC-private predictions to CINIC-ImageNet examples. Trained on 210,000 private predictions, the distilled model achieves 91.86\% accuracy, while provably bounding MIA success to 50.49\%, matching $(0.0198, 10^{-5})$-DP.
SotA DP training on CIFAR-10 with public ImageNet access achieves 94.7\% accuracy at $\epsilon=1$~\cite{de2022unlocking}.
}
\label{tab:distillation}
{
\setlength{\tabcolsep}{2.5pt}
\begin{tabular}{lccccccc}
\toprule
Number of Private Predictions ($T$) & 5,000 & 10,000 & 20,000 & 40,000 & 80,000 & 160,000 & 210,000 \\
\midrule
Total MI Budget ($B_T$) &  $1.16 \times 10^{-6}$ &  $2.33 \times 10^{-6}$ &  $4.66 \times 10^{-6}$ &  $9.31 \times 10^{-6}$ &  $1.86 \times 10^{-5}$ &  $3.73 \times 10^{-5}$ &  $4.89 \times 10^{-5}$ \\
MIA Success Rate Bound (\%) & 50.08 & 50.11 & 50.15 & 50.22 & 50.31 & 50.43 & 50.49 \\
DP $\epsilon$ w/ Matching MIA Guarantee ($\delta=10^{-5}$) & 0.0030 & 0.0043 & 0.0061 & 0.0086 & 0.0122 & 0.0172 & 0.0198 \\
\midrule
Avg. Number of Confident Predictions & 1,989 & 3,976 & 7,951 & 15,983 & 32,070 & 64,136 & 84,295 \\
Avg. Distilled Model Test Accuracy (\%) & 64.01 & 69.79 & 81.34 & 84.94 & 87.74 & 91.19 & 91.86 \\
\bottomrule
\end{tabular}

}
\end{table*}

\subsection{Defense against MIA}\label{sec:mia_experiments}

To empirically evaluate the privacy of our algorithm, we run membership inference attacks under our threat model. Standard MIA literature typically assumes an adversary with limited knowledge and thus often rely on heuristics, such as training shadow models to \emph{estimate} the behavior of the target model and setting threshold-based criteria for membership \cite{shokri2017membership, carlini2022membership}.
However, under our threat model, where the strong adversary possesses full knowledge of the distribution $P_S$ of the secret and functions $M_t$, heuristic approximation is not necessary.
The adversary can derive the \emph{optimal} strategy to maximize their success rate for a given interaction history.

\begin{proposition}[Optimal MIA Decision Rule]\label[proposition]{prop:optimal_mia}
    Let random variable $\Tau_T$ be the history of interactions up to time $T$. For a target individual $u^* \in U$, the decision rule $A: \mathrm{supp}(\Tau_T) \to \{0,1\}$ that maximizes the adversary's membership inference accuracy is:
    $$\label{eq:mia_rule}
    A^*(\tau_T) = \begin{cases} 
      1 & \text{if } \sum\limits_{S \in \Sspace} \mathbf{1}_{u^* \in S} \cdot P(S \mid \Tau_T=\tau_T) > \frac12 \\
      0 & \text{otherwise}
    \end{cases}.$$
\end{proposition}

\begin{proof}This follows the derivation of a Bayes optimal statistical procedure. Full proof is deferred to \cref{app:optimal_mia_proof}.\end{proof}

\begin{remark}\label[remark]{rm:optimal_infer_vs_optimal_query}
\cref{prop:optimal_mia} defines the optimal \emph{membership inference} strategy after observing $\Tau_T=\tau_T$. It does not, however, describe the optimal strategy for \emph{choosing} the queries $q_t$. While the adversary can adaptively select $q_t$ to maximize the expected information gain, solving for the optimal query sequence is computationally intractable.
\end{remark}

In our experiments, the adversary seeks to simultaneously infer the membership of all examples in the universe $U$, i.e., to decide whether $u\in U$ belongs to the realized secret set $S$ used for predictions.
Following standard MIA protocols~\cite{shokri2017membership, carlini2022membership}, the adversary queries the mechanism using the exact examples from $U$ whose membership they wish to discover.
Crucially, because the adversary has full system knowledge and $m=|\Sspace|$ is tractable, they can evaluate $A^*(h_T)$ exactly. This step mirrors the belief state update performed by the curator in \cref{alg:pacpac}.
We execute 1,000 independent trials. In each trial, we submit a sequence of up to one million queries, randomly sampled from $U$ without replacement. At various time steps $T$, we run the optimal decision rule $A^*(h_T)$ to infer the membership of all examples in $U$ and report the average accuracy.

MIA results on CIFAR-10 are summarized in \cref{tab:mia}.
Under $2^{-24}\leq b\leq 2^{-12}$, we observe a clear improvement in the empirical MIA accuracy with more queries. When the per-step MI budget is vanishingly small, i.e., $b=2^{-28}$ or $2^{-32}$, the MIA adversary struggles to achieve non-trivial accuracy even after one million queries.
Notably, across all settings, the empirical MIA accuracy is \emph{strictly bounded} by the theoretical guarantees derived by PAC privacy, often with a significant gap.
For instance, 
at $b=2^{-12}$, any query volume above 10,000 results in a vacuous upper bound of 100\%, yet the adversary only achieves perfect membership inference after 600,000 queries.
This parallels findings in DP, where larger $\epsilon$ often effectively defends against practical MIAs while providing loose upper bounds~\cite{jayaraman2019evaluating,lowy2024doesdifferentialprivacylarge}.

This disparity between the MIA upper bound and its empirical accuracy can be attributed to the adversary's suboptimal querying strategy (cf. \cref{rm:optimal_infer_vs_optimal_query}).
While querying examples from $U$ is a standard practice in MIA literature, an adversary could use more complex strategies, such as adversarial examples~\cite{yoqo24,oslo24}.
We note that our theoretical guarantee provides provable privacy bounds under \emph{any possible querying strategy}.

\subsection{Private Model Distillation}\label{sec:distillation_experiments}

To evaluate the distillation protocol described in \cref{sec:distillation}, we simulate a scenario where the curator uses the private prediction mechanism to label a large auxiliary dataset. As a proof of concept, we use the ImageNet partition of the CINIC-10 dataset~\cite{darlow2018cinic10imagenetcifar10} as our auxiliary dataset; while it shares the same 10 classes as CIFAR-10, it represents a distinct distribution.
This mimics realistic deployment settings: the auxiliary data may be collected from user queries, or curated from public sources.
We restrict this experiment to CIFAR-10; to the best of our knowledge, there are no similar auxiliary datasets for the other datasets.
We use all 210,000 examples in CINIC-ImageNet to query the private prediction mechanism trained on CIFAR-10, under a strict per-step budget of $b=2^{-32}$.
To ensure label quality, we apply confidence filtering (cf. \cref{prop:filtering}) with $\alpha=0.01$. The accepted samples, with their privacy-preserving labels, are then used to train a Wide-ResNet-28-10 model from scratch. We run 5 times and report the average.

The private prediction mechanism trained on CIFAR-10 achieves 58.66\% accuracy on the ImageNet examples, much lower than its 87.79\% accuracy on CIFAR-10, as a result of distribution shift. However, our confidence filtering proves highly effective in practice: it retains 40.14\% of the privately labeled auxiliary examples, and the accuracy on this subset rises dramatically to over 95\%. This confirms that the hypothesis test is able to isolate high-quality, in-distribution examples from the noisy, out-of-distribution data.

The accuracy of the distilled model on CIFAR-10 test examples is summarized in \cref{tab:distillation}.
First, we notice that the distillation utility is strong even with limited auxiliary data. With only 5,000 queries, the student trained from scratch on merely 1,989 confident examples already achieves 64.01\% accuracy. Using all 210,000 queries, the distilled model achieves 91.86\% accuracy. 
Notably, this even outperforms the private teacher (87.79\%), as the student effectively learns from the diverse features in the ImageNet examples with high-quality private labels.
This highlights that proper distillation can unlock utility beyond private prediction's baseline.

From a privacy perspective, releasing this student model exposes only the information leaked through the 210,000 private predictions about membership of the CIFAR-10 examples. The accumulated MI bound corresponds to an upper bound of 50.49\% on MIA success, matching the MIA guarantee offered by $\epsilon \approx 0.02$ under DP.
We compare this to the standard setting of differentially private training with access to public data. The state-of-the-art result under similar settings achieves 94.7\% accuracy on CIFAR-10 at $\epsilon=1$~\cite{de2022unlocking} when fine-tuning a Wide-ResNet-28-10 model pre-trained on ImageNet~\cite{imagenet}. Our method achieves slightly worse utility (91.86\%) while offering a MIA-matching DP $\epsilon$ that is \emph{over an order of magnitude tighter}. 
Further, our distillation protocol does not require the public dataset to be labeled,
which is particularly favorable in scenarios where large-scale, unlabeled public data is available.

\section{Related Work}\label{sec:related}

\vspace{-1ex}
\paragraph{PAC Privacy} \citet{hanshen2023crypto} establishes the theoretical foundations of PAC privacy, showing that it allows for the privatization of arbitrary black-box functions via simulations. Subsequent works apply PAC privacy for learnable encryption~\cite{hanshen2024ccs} and classical ML algorithms~\cite{mayuri2025sp}. The latter explicitly links utility under PAC privacy to \emph{algorithmic stability}. More recently, \cite{xiao2025onesidednoise} studies one-sided noise for PAC privacy, while \cite{zhang2026residual} uses $f$-divergence to achieve tighter noise calibration.
However, these applications operate under a single release. Composition under PAC privacy is first explored by \citet{hanshen2023crypto} with a simple linear composition theorem under the \emph{non-adaptive} setting. 
\citet{hanshen2025thesis} proves a linear adversarial composition theorem where the secret is \emph{resampled} for every release, and also achieves adversarial composition in the persistent-secret setting. However, the latter effectively enforces a DP-like input-independent bound to handle adaptivity, negating the core benefits of PAC privacy.
Our work closes this gap with tight, instance-based composition under the persistent-secret, adaptive, and adversarial setting. 

\paragraph{DP vs PAC Privacy} DP is the state-of-the-art framework for provable privacy guarantees. DP bounds the influence of a single record in the dataset, thus limiting what can be inferred about any individual from the released output. This is usually achieved by injecting noise to the output, calibrated to the algorithm's \emph{sensitivity}~\cite{dwork06dp} --- the maximum change in the output caused by a single data point.

DP and PAC privacy provide semantically different privacy guarantees. While DP provides \emph{stronger, input-independent} guarantees, PAC privacy provides \emph{instance-based} ones under a specified input distribution~\cite{hanshen2023crypto}. However, we believe PAC privacy has some advantages in practice. First, it enables \emph{automatic privatization}, whereas bounding sensitivity for DP requires substantial human efforts.
Second, PAC privacy often enjoys better privacy-utility trade-offs by calibrating noise to the algorithm's \emph{stability}.
Last, PAC privacy admits a success-rate bound for an \emph{arbitrary} attack. For example, as discussed earlier, at a fixed MI, PAC privacy bounds group membership inference more tightly than individual due to a lower prior, whereas DP requires group privacy, which increases $\epsilon$.
We now review two lines of work in differentially private machine learning that are most relevant to ours.

\paragraph{DP Training}
Conventionally, privacy-preserving ML focuses on releasing model weights with DP guarantees, most commonly via DP-SGD, which clips per-sample gradients to bound sensitivity and injects noise during optimization~\cite{song2013dpsgd,abadi2016deep}. 
However, this approach often incurs substantial utility loss, particularly on larger models. 
A recent survey~\cite{demelius25} attributes the difficulty of DP in deep learning to the model's high-dimensionality, while recent analysis~\cite{ertan2026fundamental} identifies a worst-case noise floor on DP-SGD that precludes strong privacy and high utility together.
As a result, competitive utility under DP is often achieved either with relaxed privacy budgets, or relying on public data.
This is evident in private image classification.
\citet{tramer2021dp} shows that lower dimensional classifiers on fixed features often outperform end-to-end DP training. \citet{de2022unlocking} demonstrates that scaling and careful tuning improve DP-SGD performance, yet strong utility at tight budgets ($\epsilon\approx1$) is achieved primarily via fine-tuning pretrained models. 
More recently, \citet{tang2025differentially} pushes the training-from-scratch utility of DP-SGD by learning priors generated from random processes. This achieves the 
SoTA accuracy of 72.32\% on CIFAR-10 at $\epsilon=1$ and 43.33\% on CIFAR-100 at $\epsilon=3$, still far from non-private baselines. Finally, a recent benchmark~\cite{mokhtari2025dpbenchmark} explicitly separates settings with and without public data and highlights that large public pretraining is crucial for competitive utility. %

\paragraph{DP Prediction}
A parallel line of work explores \emph{private prediction}, where DP is enforced on the sequence of outputs rather than the model itself.
\citet{dwork2018prediction} formalizes the problem and connects algorithmic stability with privacy.
However, it relies on strong assumptions to bound sensitivity, only provable for simple or convex learners.
For non-convex predictors, sensitivity cannot be tightly bounded, and the canonical approach is \emph{sample-and-aggregate}~\cite{nissim2007smooth}.
PATE~\cite{pate,papernote2017pate} adopts this approach for private prediction, aggregating the predictions of $K$ teacher models each trained on a disjoint subset of the data.
However, it is most effective for semi-supervised learning rather than as a standalone prediction service to answer a large number of queries, as we show in \cref{sec:exp:ppp}.
We further note that PATE's privacy analysis is data-dependent, with per-query loss determined by teacher vote margins, making the total number of answerable queries hard to bound a priori. 
More broadly, private prediction has seen limited adoption, as private training has been found to outperform it across many settings~\cite{van2020trade}. 
This calculus shifts for large language models (LLM), however, where privatizing weights is impractical and the prediction is itself the unit of deployment; private prediction has accordingly gained traction for LLM inference~\cite{submix,pmixed}.

\section{Conclusion}\label{sec:conclusion}

This work presented a novel framework to release PAC-private predictions under adaptive and adversarial queries. We derived a novel adversarial composition theorem for PAC privacy that achieves tight privacy accounting via \emph{adaptive noise calibration} that adjusts to an optimal adversary's shifting beliefs. Instantiated for ML predictions, our method exploits the inherent stability of model predictions and maintains high utility under tight privacy budgets.

Potential future directions include theoretical extensions of our composition theory, such as privacy accounting under $f$-divergence and efficient estimation for intractable secret spaces.
Another promising avenue is applying PAC-private prediction to LLM inference to generate longer sequences of tokens with strong utility. Finally, \cref{alg:pacpac} is a general-purpose tool applicable to any iterative algorithm with sequential, adaptive releases, such as SGD.

\bibliographystyle{IEEEtran}
\bibliography{ref}

\appendices
\section{Analysis of Posterior-Oblivious Composition}\label[appendix]{app:static_composition}

We hereby provide a more detailed analysis of the posterior-oblivious adversarial composition~\cite{hanshen2025thesis} briefly discussed in \cref{sec:pacpac}. We first state the formal theorem.

\begin{theorem}[Posterior-Oblivious Adversarial Composition~\cite{hanshen2025thesis}]\label{thm:thm4_2}
    Consider the same setting described in \cref{sec:composition_setting}. If at each step $t=1,2,\ldots, T$, the noise covariance matrix $\Sigma_t$ is calibrated w.r.t. $P_S$ and $M_t$ to enforce:\begin{gather}
        \sup_{Q} I_{S\sim Q}(S;M_t(S)+\Normal(0,\Sigma_t))\leq b'_t,\label{eq:universal_bound}\\
        I_{S\sim P_S}(S;M_t(S)+\Normal(0,\Sigma_t))\leq b_t,\label{eq:static_bound}
    \end{gather}
    Then, the overall MI after $T$ steps is bounded as follows:
    \begin{equation}
        I(S;R_1,\ldots,R_T)\leq B_T,
    \end{equation}
    where $B_1=b_1$ and for $t\geq 2$, $$
    B_t=B_{t-1}+\min\{b'_t\sqrt{2B_{t-1}}+b_t, b'_t\}.
    $$
\end{theorem}
\begin{proof}
See Theorem 4.2 and Appendix A.8 of \cite{hanshen2025thesis}.
\end{proof}

We note that the first constraint \cref{eq:universal_bound} is an input-independent worst-case guarantee that is similar to the \emph{sensitivity} for local differential privacy~\cite{duchi2013ldp}, while the second constraint \cref{eq:static_bound} is the standard mutual information bound we enforce with regard to the specific choice of $P_S$ and $M_t$.

Assuming uniform budget allocation, i.e., $b_t=b$ and $b_t'=b'$ for all $t$, the recurrence relation of the cumulative MI bound is $B_t=B_{t-1}+\min\{b',b+b'\sqrt{2B_{t-1}}\}$, where $B_1=b_1$. We consider two scenarios for the choice of $b'$ relative to $b$.

First, if $b'$ is set such that $b'\approx b$, then $B_T=b'T$ scales linearly in $T$.
However, enforcing such a small $b'$ requires the noise $\Sigma_t$ to be calibrated to the worst-case capacity of the channel (i.e., \cref{eq:universal_bound}), rendering the instance-specific bound (i.e., \cref{eq:static_bound}) redundant. This ignores the stability of $M_t$ under $P_S$ and results in a significantly higher noise level, effectively reverting to a DP-style guarantee based on \emph{worst-case sensitivity} rather than \emph{average-case stability}. This negates the utility benefits of PAC privacy.

Second, if we set $b'\gg b$ to preserve utility (i.e., we rely on the $P_S$-specific \cref{eq:static_bound}), the recurrence relation is: \begin{equation}\label{eq:recurrence_case_2}
B_t=B_{t-1}+b+b'\sqrt{2B_{t-1}}.
\end{equation}
We now show that under \cref{eq:recurrence_case_2}, the total mutual information bound scales quadratically in $T$, i.e., $B_T=\Omega(T^2)$.

Since $b>0$, we have $B_t\geq B_{t-1}+b'\sqrt{B_{t-1}}$.
Taking the square root of both sides and applying the inequality $\sqrt{x+y} \geq \sqrt{x} + \frac{y}{2\sqrt{x}} - \frac{y^2}{8x\sqrt{x}}$ (due to Taylor expansion):
$$
\sqrt{B_t}\geq \sqrt{B_{t-1}+b'\sqrt{B_{t-1}}} \geq \sqrt{B_{t-1}}+\frac {b'} 2 - \frac{b'^2}{8\sqrt{B_{t-1}}}.
$$

Since $B_t$ is strictly increasing at $t$ grows, there exists some $t'$ such that $B_{t}\geq b'^2$ for all $t\geq t'$. Then, the inequality above becomes $
\sqrt{B_t}\geq\sqrt{B_{t-1}}+3b'/8$, which is an arithmetic progression for $\sqrt{B_t}$. Therefore, $\sqrt{B_T}=\Omega(b'T)$ and $B_T=\Omega(b'^2T^2)$. In other words, the overall MI bound grows \emph{quadratically} in the horizon $T$.

Therefore, under \cref{thm:thm4_2}, one must choose between two suboptimal outcomes: either calibrates excessive noise to satisfy the input-independent MI bound at the cost of utility, or accepts a quadratic composition bound where at the same level of overall privacy guarantee and noise level, much fewer releases can be allowed.
In contrast, with \cref{alg:pacpac}, we achieve linear composition with instance-based noise calibration.

\section{Deferred Full Proofs}
\subsection{Proof of Lemma \ref{lm:post}}\label[appendix]{app:lm_proof}
\begin{proof}
    We proceed by induction on the time step $t$.
    
    For $t=0$, the only possible $\tau_0$ is empty and \cref{alg:pacpac} sets $P_0=P_S=P_{S\mid \Tau_0=\emptyset}$.
    
    Assume the property holds for $t-1$. That is, for any realizable history $\tau_{t-1}$, given that $\Tau_{t-1}=\tau_{t-1}$, \begin{equation}
        \label{eq:inductive_assumption}P_{t-1}=P_{S\mid\Tau_{t-1}=\tau_{t-1}}.
    \end{equation}
    
    At time step $t$, let $\tau_{t}=(\tau_{t-1},(m_t,r_t))$ be a realizable history. We derive the density of the posterior of $S$ conditioned on $\Tau_t=\tau_t$ via Bayes' Rule: $\forall s\in\Sspace$,
    \begin{align}
    &P(s \mid \tau_t) = P(s \mid \tau_{t-1}, m_t, r_t)\nonumber\\
    &\propto P(r_t, m_t \mid s, \tau_{t-1}) \cdot P(s \mid \tau_{t-1})\nonumber\\
    &= P(r_t, m_t \mid s, \tau_{t-1})\cdot P_{t-1}(s)\label{eq:proof3.2_eq1}\\
    &=P(m_t\mid s,\tau_{t-1})\cdot P(r_t\mid m_t,s,\tau_{t-1})\cdot P_{t-1}(s)\nonumber\\
    &=P(m_t\mid \tau_{t-1})\cdot P(r_t\mid m_t,s,\tau_{t-1})\cdot P_{t-1}(s)\label{eq:proof3.2_eq2}\\
    &\propto P(r_t\mid m_t,s,\tau_{t-1})\cdot P_{t-1}(s),\nonumber%
    \end{align}
    where \cref{eq:proof3.2_eq1} is due to induction hypothesis (\cref{eq:inductive_assumption}) and \cref{eq:proof3.2_eq2} is due to the assumption in \cref{eq:adv_ass}.
    
    Conditioned on $S=s,\Tau_{t-1}=\tau_{t-1},M_t=m_t$, we have $R_t=m_t(s)+\mathcal N(0,\Sigma_t)$ where $\Sigma_t=\Sigma(P_{t-1},M_t,b_t)$ is known. Hence, $P_{R_t\mid m_t,s,\tau_{t-1}}=\Normal(m_t(s),\Sigma_t)$ and $$
    \begin{aligned}
    &P(r_t\mid m_t,s,\tau_{t-1})\\
    &\propto\exp\left(-\frac12(r_t-m_t(s))^\top\Sigma_t^{-1}(r_t-m_t(s))\right).
    \end{aligned}
    $$
    Substituting back, we have $(s \mid \tau_t)\propto$ $$
    P P_{t-1}(s)\cdot\exp\left(-\frac12(r_t-m_t(s))^\top\Sigma_t^{-1}(r_t-m_t(s))\right),$$ which matches exactly \cref{eq:p_upadte} to compute $P_t$.
\end{proof}

\subsection{Proof of Theorem \ref{thm:composition}}\label[appendix]{app:composition_proof}
\begin{proof}
Since 
$$
\begin{aligned}
I(S ; R_1,\dots,R_T)&\leq I(S;(R_1,M_1)\dots,(R_T,M_T))\\
&=I(S;\Tau_T),
\end{aligned}
$$ we prove a stronger guarantee $I(S;\Tau_T)\leq B_T$.

By the chain rule of mutual information, we have
$$\begin{aligned}
    \textstyle I(S;\Tau_T)&\textstyle=\sum_{t=1}^TI(S;M_t,R_t\mid \Tau_{t-1})\\
    &\textstyle=\sum_{t=1}^T I(S;M_t\mid \Tau_{t-1})+I(S;R_t\mid M_t,\Tau_{t-1})\\
    &\textstyle=\sum_{t=1}^TI(S;R_t\mid M_t,\Tau_{t-1}),
\end{aligned}$$
where the last equality is due to \cref{eq:adv_ass}.

We analyze its $t$-th term:
{
\def\myexp{\E_{\tau_{t-1}, m_t}}
\def\cond{\mid \tau_{t-1},\,m_t}
\begin{align}
&I(S;R_t\mid M_t,\Tau_{t-1})=\myexp[I(S;R_t\cond)]\nonumber\\
&=\myexp [I(S; m_t(S)\nonumber\\
&\quad\qquad\qquad+\mathcal N(0, \Sigma(P_{t-1},m_t,b_t)\cond)]\nonumber\\
&=\myexp [I(S; m_t(S)\nonumber\\
&\quad\qquad\qquad+\mathcal N(0, \Sigma(P_{S\mid\tau_{t-1}},m_t,b_t)\cond)] \label{eq:use_lm1}\\
&=\myexp [I_{S\sim P_{S\mid\tau_{t-1},m_t}}(S; m_t(S)\nonumber\\
&\quad\qquad\qquad+\mathcal N(0, \Sigma(P_{S\mid\tau_{t-1}},m_t,b_t))]\nonumber\\
&=\myexp [I_{S\sim P_{S\mid\tau_{t-1}}}(S; m_t(S)\nonumber\\
&\quad\qquad\qquad+\mathcal N(0, \Sigma(P_{S\mid\tau_{t-1}},m_t,b_t))]\label{eq:use_cond_indp}\\
&\leq\myexp[b_t]=b_t.\label{eq:use_thm1}
\end{align}
}

We justify the derivations as follows. \cref{eq:use_lm1} holds as \cref{lm:post} guarantees $P_{t-1} = P_{S \mid \Tau_{t-1}=\tau_{t-1}}$ when conditioned on $\Tau_{t-1}=\tau_{t-1}$, and the additional conditioning on $M_t$ in the prior is redundant due to the adversary's conditional independence (cf. \cref{eq:adv_ass}). This independence also implies \cref{eq:use_cond_indp}, as it allows us to drop $M_t$ from the conditioning in the probability measure. Finally, \cref{eq:use_thm1} follows directly from \cref{def:noise}.

Summing over $t$ recovers the statement.
\end{proof}

\subsection{Proof of Proposition \ref{prop:filtering}}\label[appendix]{app:filter_proof}

\begin{proof}%
Under $y\neq \tilde y$, the label is altered by noise, and the probability of retaining this mislabelled sample $(x,\tilde y)$ is
$$\Pr(\min_{j\neq\tilde y} T_j(r)\geq \Phi^{-1}(1-\alpha))\leq \Pr(T_y(r)\geq\Phi^{-1}(1-\alpha)).$$

Here, $T_y(r)$ is $$
\begin{aligned}
t&=\frac{(e_{\tilde y}-e_y)^\top\Sigma^{-1}(r-e_y)}{\sqrt{(e_{\tilde y}-e_y)^\top\Sigma^{-1}(e_{\tilde y}-e_y)}}\\
&=\frac{(e_{\tilde y}-e_y)^\top\Sigma^{-1}}{\sqrt{(e_{\tilde y}-e_y)^\top\Sigma^{-1}(e_{\tilde y}-e_y)}}(r-e_y):=A(r-e_y),
\end{aligned}
$$ where $r\sim\Normal(e_y,\Sigma)$. Therefore, $t\in\Real$ is a linear transformation of $(r-e_y)\sim\Normal(0,\Sigma)$, and hence follows a normal distribution with zero mean, whose variance is:\begin{align*}
A\Sigma A^\top
&=\frac{(e_{\tilde y}-e_y)^\top\Sigma^{-1}}{(e_{\tilde y}-e_y)^\top\Sigma^{-1}(e_{\tilde y}-e_y)}\Sigma (\Sigma^{-1})^\top (e_{\tilde y}-e_y)\\
&=\frac{(e_{\tilde y}-e_y)^\top\Sigma^{-1}(e_{\tilde y}-e_y)}{(e_{\tilde y}-e_y)^\top\Sigma^{-1}(e_{\tilde y}-e_y)}=1.
\end{align*}

Therefore, $t\sim\Normal(0,1)$, and $$
\Pr(\min_{j\neq\tilde y} T_j(r)\geq \Phi^{-1}(1-\alpha))\leq \Pr(t\geq\Phi^{-1}(1-\alpha))=\alpha,
$$ which proves the statement.\end{proof}

\subsection{Proof of Proposition \ref{prop:optimal_mia}}\label[appendix]{app:optimal_mia_proof}
\begin{proof}%
Let $Y = \mathbf{1}_{u^* \in S}$ be the true membership status. Due to the balanced construction of $P_S$, we have $\Pr(Y=0)=\Pr(Y=1)=0.5$. Let $A:\mathrm{supp}(\Tau_T)\to\{0,1\}$ be a membership inference decision rule. Then its accuracy is
\begin{align*}
\text{Accuracy} & =\E_{y}[\E_{\tau_T}[\mathbf 1_{A(\tau_t)=y}\mid Y=y]]\\
&=\E_{\tau_T}[ \E_y[\mathbf 1_{A(\tau_t)=y}\mid \Tau_T=\tau_T] ]\\
&\leq \E_{\tau_T}[
\max_{a\in\{0,1\}} \E_y[\mathbf 1_{a=y}\mid \Tau_T=\tau_T]
].
\end{align*}

Therefore, the following decision rule is optimal: \begin{align*}
A(\tau_T)&=\argmax_{a\in\{0,1\}}\E_y[\mathbf 1_{a=y}\mid\Tau_T=\tau_T]\\
&=\argmax_{a\in\{0,1\}} (1/2\cdot \mathbf 1_{a=0} \Pr(Y=0\mid\Tau_T=\tau_T)\\
&\qquad\qquad+ 1/2\cdot\mathbf 1_{a=1} \Pr(Y=1\mid\Tau_T=\tau_T)).
\end{align*}
This is equivalent to $A(\tau_T)=1$ if and only if \begin{gather*}
    \Pr(Y=0\mid\Tau_T=\tau_T) < \Pr(Y=1\mid\Tau_T=\tau_T)\\
    \Leftrightarrow \Pr(Y=1\mid\Tau_T=\tau_T)>0.5,
\end{gather*} where \begin{align*}
\Pr(Y=1\mid\Tau_T=\tau_T) = \textstyle\sum_{S\in\Sspace} \mathbf1_{u^*\in S} \cdot P(S\mid\Tau_t=\tau_t).
\end{align*}
This recovers the decision rule $A^*$ in the statement.
\end{proof}

\begin{table*}[t]
\caption{Average test accuracy (\%) of PAC-private predictions under various per-query mutual information budgets $2^{-32}\leq b\leq 2^{-4}$ across datasets when privatizing the predicted confidence scores vs one-hot vectors.}
\label{tab:score_vs_onehot}
\centering
\setlength{\tabcolsep}{2.5pt}
\begin{tabular}{llccccccccccc}
\toprule
Modality & Dataset & Non-Private & Output & $b=\infty$ & $b=2^{-4}$ & $b=2^{-8}$ & $b=2^{-12}$ & $b=2^{-16}$ & $b=2^{-20}$ & $b=2^{-24}$ & $b=2^{-28}$ & $b=2^{-32}$ \\
\midrule
\multirow{4.2}{*}{Tabular} & \multirow{2}{*}{Census Income} & \multirow{2}{*}{87.39} & Confidence & 87.17 & 87.16 & 86.59 & 74.72 & 65.55 & 57.57 & 52.40 & 50.62 & 50.16 \\
 &  &  & One-Hot & 87.17 & 87.15 & 86.68 & 85.92 & 85.86 & 85.84 & 85.84 & 85.84 & 85.84 \\
\cmidrule{2-13}
 & \multirow{2}{*}{Bank Marketing} & \multirow{2}{*}{91.98} & Confidence & 91.69 & 91.69 & 91.26 & 84.52 & 77.64 & 62.44 & 53.33 & 50.86 & 50.19 \\
 &  &  & One-Hot & 91.69 & 91.67 & 91.02 & 90.36 & 90.28 & 90.28 & 90.27 & 90.27 & 90.29 \\
\midrule
\multirow{4.2}{*}{Image} & \multirow{2}{*}{CIFAR-10} & \multirow{2}{*}{97.37} & Confidence & 95.80 & 95.79 & 93.80 & 45.73 & 35.86 & 34.04 & 33.57 & 33.45 & 33.41 \\
 &  &  & One-Hot & 95.80 & 95.71 & 93.52 & 88.35 & 87.90 & 87.81 & 87.80 & 87.80 & 87.79 \\
\cmidrule{2-13}
 & \multirow{2}{*}{CIFAR-100} & \multirow{2}{*}{84.02} & Confidence & 77.79 & 77.72 & 74.36 & 35.79 & 32.36 & 31.76 & 31.66 & 31.63 & 31.61 \\
 &  &  & One-Hot & 77.79 & 77.69 & 75.56 & 58.38 & 56.34 & 56.17 & 56.13 & 56.10 & 56.11 \\
\midrule
\multirow{4.2}{*}{Text} & \multirow{2}{*}{IMDb Reviews} & \multirow{2}{*}{87.10} & Confidence & 85.13 & 85.12 & 84.69 & 71.91 & 55.07 & 51.28 & 50.32 & 50.09 & 50.04 \\
 &  &  & One-Hot & 85.13 & 85.13 & 84.46 & 74.26 & 69.32 & 69.16 & 69.09 & 69.10 & 69.10 \\
\cmidrule{2-13}
 & \multirow{2}{*}{AG News} & \multirow{2}{*}{91.61} & Confidence & 90.44 & 90.39 & 88.86 & 62.88 & 45.41 & 39.99 & 38.67 & 38.34 & 38.29 \\
 &  &  & One-Hot & 90.44 & 90.35 & 87.95 & 80.18 & 79.42 & 79.31 & 79.27 & 79.25 & 79.25 \\
\bottomrule
\end{tabular}
\end{table*}

\section{Additional Attacks}\label[appendix]{app:attacks}

We establish the prior success rate $\bar\delta_0$ of an informed adversary for several attacks under the input distribution $P_S$ constructed in \cref{sec:input_discribution}: $P_S$ is uniform over $\Sspace=\{S_1,\ldots,S_m\}$, with each record contained in randomly selected $m/2$ of the $m$ subsets. For each attack we cast it as a criterion $\rho$ as in \cref{def:pac} and compute the optimal prior --- the highest success an informed adversary attains \emph{before} observations. Combined with the total MI budget $B_T$, each prior yields a posterior success bound via \cref{eq:delta<=mi}.

\begin{figure}[t]
\centering
\includegraphics{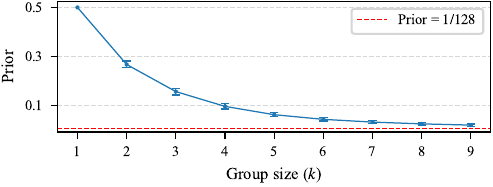}
\caption{Average prior success rate $\bar\delta_0$ of group MIA as a function of the group size $k$, under our construction of $\Sspace$ with $m=128$; error bars show standard deviation over random subset constructions.}
\label{fig:group_mia_prior}
\end{figure}

\paragraph{Group MIA}
Fix a $k$-tuple of records $\vec u=(u_1,\ldots,u_k)\in U^k$. The adversary guesses the tuple's joint membership configuration and succeeds iff the guess matches the realized configuration:
$\rho(\hat S,S)=1$ iff $(\mathbf 1_{u_j\in\hat S})_{j=1}^k=(\mathbf 1_{u_j\in S})_{j=1}^k.$
Since $S$ is uniform over $\Sspace$, $(\mathbf 1_{u_j\in S})_{j=1}^k$ follows the empirical distribution of the $m$ configurations $\{(\mathbf 1_{u_j\in S_i})_{j=1}^k\}_{i=1}^m$, and the optimal prior is the frequency of the most common configuration,
$\bar\delta_0=\max_{c\in\{0,1\}^k}\frac{1}{m}|\{\,i:(\mathbf 1_{u_j\in S_i})_{j=1}^k=c\}|$.
This quantity depends on the realized construction of $\Sspace$, and its expectation over the random construction admits no closed form; we estimate it by simulation. For $k=1$ the configuration is a single bit set in exactly $m/2$ subsets, so $\bar\delta_0=1/2$, recovering individual MIA. As $k$ increases, the configurations spread over up to $2^k$ values and the prior decreases toward the $1/m$ floor, as shown in \cref{fig:group_mia_prior}.
In contrast, DP must revert to group privacy and increase $\epsilon$ to bound group MIA success rates.

\paragraph{Individual Example Reconstruction}
The adversary outputs a single record $\hat x$ and succeeds iff it belongs to the secret training set, i.e.,
$\rho(\hat x,S)=\mathbf 1_{\hat x\in S}$.
Because every record lies in exactly $m/2$ of the $m$ subsets, $\Pr_{S\sim P_S}[\hat x\in S]=1/2$ for every $\hat x\in U$, so the informed-adversary prior is $\bar\delta_0=1/2$. This rate is tied to our sampling rate: a construction in which each record appears in fewer subsets lowers it, at the cost of moving the membership prior away from $50\%$. PAC privacy further admits a relaxed criterion that concerns approximate reconstruction, e.g.,
$\rho(\hat x,S)=\mathbf 1_{\exists\,x'\in S: \mathrm{dist}(\hat x,x')\le t}$
for a threshold $t$; the corresponding prior is data-dependent, governed by the geometry of $U$, and can likewise be measured under $P_S$. We note that DP can similarly provide provable guarantees on reconstruction~\cite{balle2022reconstructing} and recent work~\cite{hanshen2024ccs} shows that PAC privacy does so with better utility.

\paragraph{Full Training-Set Reconstruction}
The adversary outputs a candidate set $\hat S$ and succeeds iff it equals the secret, i.e.,
$\rho(\hat S,S)=\mathbf 1_{\hat S=S}$.
As $S$ is uniform over $m$ candidates, the optimal prior is $\bar\delta_0=\max_i\Pr[S=S_i]=1/m$. This is the floor of \cref{prop:floor}: no attack where identification of the secret constitutes success can have a lower prior.

\paragraph{Model Extraction}
The adversary outputs a model $g:\mathcal X\to\mathcal Y$ and succeeds iff it matches the model trained on the secret, i.e., 
$\rho(g,S)=\mathbf 1_{g\equiv\hat f(\,\cdot\,;S)}$,
where $\equiv$ denotes equality of weights or functionality, depending on the attack definition~\cite{tramer2016stealing}. Each candidate $S_i$ induces a model $\hat f(\,\cdot\,;S_i)$, and the optimal prior is the frequency of the most common model among $\{\hat f(\,\cdot\,;S_i)\}_{i=1}^m$. When these $m$ models are distinct --- as is typically the case, since training on different subsets yields different weights and decision boundaries --- each occurs once and the prior is $\bar\delta_0=1/m$, again the floor.

All priors above are computed for the \emph{informed} adversary of our threat model, who knows $P_S$ and the realized construction of $\Sspace$, and are therefore conservative: a realistic adversary, who does not know the specific subsets and must infer them from a high-entropy universe, has a strictly lower prior. Crucially, the 
posterior success bounds of \cref{sec:privacy_guarantees} derived from \cref{thm:post_adv_mi} remain valid for \emph{any} adversary.

\section{Implementation Details}\label[appendix]{app:implementation_details}
\vspace{-1ex}
\paragraph{Tabular Datasets} For tabular datasets, we use XGBoost~\cite{chen2016xgboost} due to its strong performance on structured data. Categorical features are handled using XGBoost's built-in automatic feature processing method, which internally applies appropriate encoding without requiring manual preprocessing.
Model selection is performed on the secret training set $S$ using five-fold cross validation. We conduct a grid search over key hyperparameters controlling model capacity, learning dynamics, and regularization. Specifically, we vary the maximum tree depth, learning rate, number of boosting iterations, row/column subsampling ratios, minimum child weight, and the regularization parameter $\gamma$. The final model is trained on $S$ using the hyperparameter configuration that achieves the best average cross-validation performance.

\paragraph{Image Datasets} For CIFAR-10 and CIFAR-100, we train a Wide-ResNet-28-10~\cite{zagoruyko2016wide} model from scratch. All experiments are conducted using standard data preprocessing and augmentation protocols. Images are randomly cropped with padding and horizontally flipped, followed by normalization. We use ImageNet statistics for normalization, which does not introduce privacy leakage. For stronger data augmentation, we additionally employ CutMix~\cite{yun2019cutmix} and MixUp~\cite{zhang2018mixup}, where one of the two is randomly applied on each batch.
The model is trained without dropout. Optimization is performed with SGD, with momentum set to 0.9 and weight decay set at $5\times10^{-4}$ for regularization. The initial learning rate is set to 0.1 and decayed using a multi-step schedule, with learning rate drops by 20\% at 60 and 120 epochs. All models are trained for a total of 200 epochs. 
Mixed-precision (FP16) training is enabled to improve efficiency.

\paragraph{Text Datasets} For NLP tasks, we use train a BERT-Small~\cite{devlin2019bert,turc2019well} from scratch for text classification. The model consists of 4 transformer encoder layers with a hidden dimension of 512, 8 attention heads per layer, and an intermediate feed-forward dimension of 2048. Both hidden-state and attention dropout are set to 0.1. The maximum sequence length is 512 tokens.
We use a batch size of 64 and train for up to 50 epochs with early stopping, which monitors the accuracy on a reserved 10\% of the training set $S$. Optimization is done with AdamW, using a baseline learning rate of $10^{-4}$, linear warm-up for the first 500 steps, and weight decay set to 0.01 for regularization. Mixed-precision (FP16) training is enabled to improve efficiency.

\paragraph{DP Baseline}
For our DP prediction baseline, 
for each of the $K$ disjoint subsets, we train a model with the same architecture, hyperparameter search configuration, and training recipe as our experiments for PAC privacy.
We then run DP prediction for $K\in\{25,50,100,200,250,500\}$, each for 1000 trials, and give it the advantage to select the best-performing $K$ \emph{at no privacy cost}. We use the GNMax mechanism~\cite{papernote2017pate} to noise the aggregate of the $K$ model's hard-label predictions and use exact GDP composition to calibrate the noise scale $\sigma$ to achieve $(\epsilon,\delta)$-DP after varying number of queries. We do not adopt data-dependent accounting~\cite{papernote2017pate,pate}, because our PAC privacy guarantee after $T$ queries is not data-dependent. That is, the resulting privacy guarantee (i.e., the total MI and the attack success bound) is specified \emph{a priori} and does not depend on online data.

\paragraph{Distillation} For our proof of concept for the proposed privacy-preserving distillation protocol, we use CINIC-10~\cite{darlow2018cinic10imagenetcifar10} as an auxiliary dataset. In particular, we only use the ImageNet subset of this dataset to avoid data leakage on CIFAR-10. This constitutes of a total of 210,000 examples. When training the distillation model, we use the same model architecture, i.e., Wide-ResNet-28-10, and the same training recipe as we train the models on subsets of CIFAR-10, with the addition of balanced sampling as the filtered, privately labeled distillation dataset is not guaranteed to be balanced.

\section{Additional Experimental Results}\label[appendix]{app:results}
\noindent\textbf{Effects of Output Stability.}
\cref{tab:score_vs_onehot} presents the test accuracy of PAC-private predictions when privatizing the predicted confidence scores compared to one-hot vectors.
We note that when the per-step MI budget $b$ is relatively large, i.e., $b\geq 2^{-8}$, privatizing confidence scores leads to slightly better utility because the required noise level is low while confidence scores provide more information than the one-hot predictions. However, when the per-step MI budget tightens, the utility of the PAC-private confidence scores rapidly drops to a much lower level compared to the one-hot predictions. In particular, when $b=2^{-32}$, for the three binary classification tasks, i.e., Census Income, Bank Marketing, and IMDb Reviews, its accuracy collapses to chance. This is because the confidence scores are much less stable than the one-hot predictions, which leads to excessive amounts of noise required to hide the secret set $S$ when the per-step privacy budget $b$ approaches zero.
At last, we note that the DP prediction baseline already privatizes hard-label predictions while still unable to exploit their inherent stability (cf. \cref{fig:pac_vs_dp}).

\end{document}